\documentclass[twoside, letterpaper, 12pt]{article}

\DeclareMathSizes{10}{10}{10}{10}
\usepackage[left=2cm,right=2cm,top=1.5cm,bottom=2.5cm,
footskip=1.5cm,headsep=1.5cm]{geometry}
\usepackage{fancyhdr}
\usepackage[utf8]{inputenc} 
\usepackage[T1]{fontenc}    
\usepackage{url}            
\usepackage{booktabs}       
\usepackage{amsfonts}       
\usepackage{nicefrac}       
\usepackage{microtype}      

\usepackage{subfigure}
\usepackage{caption}
\usepackage{graphics} 
\usepackage{epsfig} 
\usepackage{amsmath} 
\usepackage{amssymb}  
\usepackage{makeidx}
\usepackage{float}
\usepackage{balance}
\usepackage{bbm}
\usepackage{mathrsfs}
\usepackage{enumitem}
\setlist[enumerate]{itemsep=0mm}
\usepackage{multicol}
\usepackage{algorithm}
\usepackage{algpseudocode}
\usepackage{dsfont}

\usepackage{mathrsfs}  
\usepackage[normalem]{ulem}
\usepackage{xr}
\usepackage{xcolor}

\usepackage[colorlinks]{hyperref}       
\hypersetup{
  colorlinks = true,
  allcolors = siaminlinkcolor,
  urlcolor = siamexlinkcolor,
}
\colorlet{siaminlinkcolor}{green!50!black}
\colorlet{siamexlinkcolor}{red!50!black}

\newcommand{\rd}{{\mathrm d}}

\newcommand{\vx}{{\bf x}}
\newcommand{\vd}{{\bf d}}

\newcommand{\vz}{{\bf z}}

\newcommand{\vr}{{\bf r}}
\newcommand{\vs}{{\bf  s}}
\newcommand{\vf}{{\bf f}}

\newcommand{\vv}{{\bf  v}}

\newcommand{\vM}{{\bf M}}

\newcommand{\vPhi}{{\mbox{\boldmath$\Phi$}}}

\newcommand{\calH}{{\cal H}}
\newcommand{\calD}{{\cal D}}
\newcommand{\calF}{{\cal F}}
\newcommand{\calN}{{\cal N}}
\newcommand{\calL}{{\cal L}}
\newcommand{\calP}{{\cal P}}

\newcommand{\calW}{{\cal W}}

\newcommand{\argmin}{\operatornamewithlimits{argmin}}

\newcommand{\Qb}{\mathbb{Q}}
\newcommand{\Pb}{\mathbb{P}}

\newcommand{\Rb}{\mathbb{R}}
\newcommand{\Eb}{\mathbb{E}}

\newcommand{\Zb}{\mathbb{Z}}

\newcommand{\A}{\mathscr{A}}

\newcommand{\vm}{{\bf m}}

\newcommand{\calV}{{\cal V}}
\newcommand{\calB}{{\cal B}}
\newcommand{\calE}{{\cal E}}

\newcommand{\vTheta}{{\mbox{\boldmath$\Theta$}}}

\newcommand{\Langle}{{\Big\langle}}
\newcommand{\Rangle}{{\Big\rangle}}

\makeatletter
\newcommand{\vast}{\bBigg@{3.5}}
\newcommand{\Vast}{\bBigg@{4}}
\newcommand{\vastt}{\bBigg@{4.5}}
\newcommand{\Vastt}{\bBigg@{5}}
\makeatother

\makeatletter
\newcommand{\mysub}[2][]
{\bgroup
  \sbox0{#2}\usebox0
  \dimen0=\dimexpr \textwidth-\wd0\relax
  \ifx\\\@centercr \divide\dimen0 by 2\fi
  \sbox1{\begin{minipage}[t]{\dimen0}
    \subcaption{#1}%
  \end{minipage}}%
  \rlap{\raisebox{\dimexpr \ht0-\ht1}[0pt][0pt]{\usebox1}}\allowbreak
\egroup}
\makeatother

\usepackage[nolist]{acronym}
\begin{acronym}
\acro{DNN}{Deep Neural Network}
\acro{ODE}{Ordinary Differential Equation}
\acro{SPDE}{Stochastic Partial Differential Equation}
\acro{FNN}{Feed-forward Neural Network}
\acro{CNN}{Convolutional Neural Network}
\acro{DP}{Dynamic Programming}
\acro{LSTM}{Long-Short Term Memory}
\acro{FC}{Fully Connected}
\acro{DDP}{Differential Dynamic Programming}
\acro{HJB}{Hamilton-Jacobi-Bellman}
\acro{PDE}{Partial Differential Equation}
\acro{PI}{Path Integral}
\acro{NN}{Neural Network}
\acro{SOC}{Stochastic Optimal Control}
\acro{RL}{Reinforcement Learning}
\acro{MPC}{Model Predictive Control}
\acro{IL}{Imitation Learning}
\acro{RNN}{Recurrent Neural Network}
\acro{DL}{Deep Learning}
\acro{RN}{Radon-Nikodym}
\acro{SGD}{Stochastic Gradient Descent}
\acro{SDE}{Stochastic Differential Equation}
\acro{VRL}{Variational Reinforcement Learning}
\acro{IDVRL}{Infinite Dimensional Variational Reinforcement Learning}
\acro{1D}{1-dimensional}
\acro{2D}{2-dimensional}
\acro{3D}{3-dimensional}
\acro{ROM}{Reduced Order Model}
\acro{STSO}{Spatio-Temporal Stochastic Optimization}
\acro{ANN}{Artificial Neural Network}
\acro{ADPL}{Actuator Design and Policy Learning}
\end{acronym}
\usepackage{amsthm}
\usepackage{cleveref}
\usepackage[square,numbers]{natbib}

\newtheorem{theorem}{Theorem}[section]

\newtheorem{proposition}[theorem]{Proposition}

\newtheorem{definition}{Definition}[section]

\catcode`\_=11\relax
    \newcommand\email[1]{\_email #1\q_nil}
    \def\_email#1@#2\q_nil{%
      \href{mailto:#1@#2}{{\emailfont #1\emailampersat #2}}
    }
    \newcommand\emailfont{}
    \newcommand\emailampersat{\small@}
    \catcode`\_=8\relax

\begin{document}
	
	\title{
	\rule{\linewidth}{4pt}\vspace{0.3cm} \Large \textbf{
	  Stochastic Spatio-Temporal Optimization for Control and Co-Design of Systems in Robotics and Applied Physics 
	}\\ \rule{\linewidth}{1.5pt}}
	\author{Ethan N. Evans$^{a,}\thanks{Corresponding Author. Email: \email{eevans41@gatech.edu}}~$, Andrew P. Kendall$^{a}$, and Evangelos A. Theodorou$^{a,b}$\\ \vspace{-0.1cm}
	\small{$^a$Georgia Institute of Technology, Department of Aerospace Engineering} \\ \vspace{-0.2cm}
	\small{$^b$Georgia Institute of Technology, Institute of Robotics and Intelligent Machines} }
	
	\date{\small{This manuscript was compiled on \today}}
	
	\maketitle

\begin{abstract}
Correlated with the trend of increasing degrees of freedom in robotic systems is a similar trend of rising interest in Spatio-Temporal systems described by Partial Differential Equations (PDEs) among the robotics and control communities. These systems often exhibit dramatic under-actuation, high dimensionality, bifurcations, and multimodal instabilities. Their control represents many of the current-day challenges facing the robotics and automation communities. Not only are these systems challenging to control, but the design of their actuation is an NP-hard problem on its own. Recent methods either discretize the space before optimization, or apply tools from linear systems theory under restrictive linearity assumptions in order to arrive at a control solution. This manuscript provides a novel sampling-based stochastic optimization framework based entirely in Hilbert spaces suitable for the general class of \textit{semi-linear} SPDEs which describes many systems in robotics and applied physics. This framework is utilized for simultaneous policy optimization and actuator co-design optimization. The resulting algorithm is based on variational optimization, and performs joint episodic optimization of the feedback control law and the actuation design over episodes. We study first and second order systems, and in doing so, extend several results to the case of second order SPDEs. Finally, we demonstrate the efficacy of the proposed approach with several simulated experiments on a variety of SPDEs in robotics and applied physics including an infinite degree-of-freedom soft robotic manipulator.
\end{abstract}

\newpage
\section{Introduction}\label{sec:intro}
In many complex natural processes, a variable such as temperature or displacement has values that are time varying on a spatial continuum over which the system is defined. These spatio-temporal processes are typically described by \acp{PDE}, are ubiquitous in nature, and are increasingly prevalent throughout the robotics community. Some particularly interesting natural processes include the stochastic Navier-Stokes equation which governs fluid flow and turbulence, the stochastic Nonlinear Schrödinger (NLS) equation, which governs the dynamics of the wavefunction of subatomic particles~\cite{bouten2004stochastic}, the stochastic Nagumo equation which governs how voltage travels across a neuron in a brain~\cite[Chapter 10]{lord_powell_shardlow_2014}, and the  stochastic Kuramoto-Sivashinsky (KS) equation which governs flame front propagation in combustion~\cite{gomes2017controlling}. 

Beyond their role in natural process, several PDE models have been applied in challenging robotics problems. Multi-agent consensus-based control has demonstrated connections to the Heat equation~\cite{ferrari2006analysis}. Swarm robotics can be described by a Burgers-like Reaction-advection-diffusion \ac{PDE}~\cite{elamvazhuthi2018pde}. Robot navigation in crowded environments can be described by Nagumo-like \acp{PDE}~\cite{aidman2008coupled}. Soft robotic limbs can be modelled as damped Euler-Bernoulli systems~\cite{shapiro2015modeling} or as a continuum of spring-mass-dampers \cite{yekutieli2005dynamic}. 

Despite their ubiquity in nature and engineering, the theory of control of \ac{SPDE} systems was only introduced in the last few decades~\cite{da2014stochastic,fabbri} and remains incomplete especially for stochastic boundary control. This is largely due to the extraordinary challenges that \acp{SPDE} present from the perspective of control. Namely, \acp{SPDE} often exhibit a significant time delay from a control signal, dramatic under-actuation, high system dimensionality, and can have numerous bifurcations and multi-modal instabilities. Furthermore, realistic representations of these systems are stochastic, and typically are modeled with a form of stochasticity that acts spatially and temporally. 

From the perspective of mathematics, the existence and uniqueness of solutions of \acp{SPDE} remains an open problem for many systems. When solutions do exist, they often have a weak notion of  differentiability if at all. Furthermore, analysis of their dynamics must be treated with a suitable calculus over functionals. Finally their state vectors are often described by vectors in an infinte-dimensional time-indexed Hilbert space, even for scalar 1-dimensional \acp{SPDE}. Put together, mathematically consistent and numerically realizable algorithms for control of spatio-temporal systems represent many of the largest current-day challenges facing the robotics and automatic control communities.

Coupled to the challenges of control are the challenges of designing an effective actuation of the system such that the system experiences, and \textit{maximizes}, the effect of some control policy. Such an actuation design problem is its own NP-hard problem due to the continua of possible actuator designs and possible placements over the spatial continuum of the domain of the \acp{SPDE}. Furthermore, actuation design performance is coupled to the performance of the control policy, and it is quite easy to confuse poor control performance with poor actuator design and vice versa. As a result, the challenge of a-priori deducing optimal actuation by a "human expert" that leverages the dynamics, even for relatively simple \acp{SPDE}, is quite daunting and often results in naive choices.

This manuscript addresses stochastic optimal control and co-design of \acp{SPDE} through the lens of stochastic optimization. We propose a joint policy network optimization and actuator co-design optimization strategy via episodic reinforcement that leverages inherent spatio-temporal stochasticity in the dynamics for optimization. The resulting stochastic gradient descent approach bootstraps off the widespread success of \ac{SGD} methods such as ADAM for training \acp{ANN}. The stochastic calculus are extended to handle second-order \acp{SPDE} in order to address continuum mechanical systems, which in their mathematical treatment resemble their second-order \ac{ODE} counterparts prevalent in mechanical systems.

This work is motivated by many of the applications of \acp{PDE} in robotics, yet primarily seeks to establish capabilities for the eventual design, fabrication, and control of soft-body robots. The behavior of such systems in general follow second order \acp{SPDE}. As such, while the proposed methods are general to first and second order systems, we focus our mathematical formulation on second order \acp{SPDE}.
\newline

\noindent\textbf{Our contributions.}
Our approach is founded on a general principle coming from thermodynamics that also has had success in stochastic optimal control literature~\cite{TheodorouCDC2012}
\begin{equation}\label{eq:Free_Energy_Relative_Entropy}
\text{Free Energy} \leq \text{Work} - \text{Temperature} \times \text{Entropy} 
\end{equation}
We leverage this principle in order to derive a measure-theoretic loss function that utilizes exponential averaging over importance sampled system trajectories in order to choose network and actuator design parameters that simultaneously minimize state cost and control effort. This work unifies our prior work in \cite{Evans2019IDVRL} and \cite{evans2020spatio}, and provides several extensions including  scaling policy and actuator co-design optimization to large 2D systems, and extending the approach to complex nonlinear 2D second-order systems that more closely resemble a soft-robotic limb. Specifically we contribute the following:
\begin{enumerate}
    \item A practical method to reduce sub-optimal convergence in actuator co-design optimization
    \item Scalability improvements which enable joint policy and co-design optimization for 2D \acp{SPDE}
    \item Demonstration of the policy and co-design optimization on a complex nonlinear 2D second-order \ac{SPDE} model of a soft-robotic limb
\end{enumerate}

\section{Related Work} \label{sec:related_work}

Despite their ubiquity, their challenging nature has caused the theory of control of \acp{SPDE} to be introduced only in the last few decades~\cite{da1994stochastic,fabbri} and remains incomplete especially for stochastic boundary control. Numerical results and algorithms for distributed control of \acp{SPDE} are limited and typically require some model reduction approach~\cite{lou2009model,gomes2017controlling}. In~\cite{StochasticBurgers_1999}, the authors approach the control of the stochastic Burgers equation through the \ac{HJB} theory by applying the linear Feynman-Kac lemma; nevertheless, it lacks numerical results. In~\cite{moura2013optimal}, the authors treat optimal control of linear deterministic \acp{PDE} by applying linear control theory, however this work is limited to linear \acp{PDE}. The book~\cite{fabbri} gives a complete understanding of our ability, so far, to apply optimal control theory to these systems.

Most notable among existing infinite dimensional control frameworks,~\cite{DaPrato1999} investigates explicit solutions to the \ac{HJB} equation for the stochastic Burgers equation based on an exponential transformation, and~\cite{feng2006} provides an extension of the large deviation theory to infinite dimensional spaces that creates connections to \ac{HJB} theory. These and most other works on \ac{HJB} theory for SPDEs mainly focus on theoretical contributions and leave literature with algorithms and numerical results tremendously sparse. Furthermore, \ac{HJB} theory for boundary control has certain mathematical difficulties which impose limitations.

Outside of infinite dimensional methods are a body of work that treats control of soft robotic manipulators through finite dimensional nonlinear control methods. These methods are typically applied to either constant curvature or piecewise constant curvature \ac{ODE} models, and often further reduce dimensionality through dynamic inversion or orthogonal projection. A complete review of these methods can be found in \cite{george2018control}.

\begin{figure}[!t]
\centering
  \includegraphics[width=0.7\linewidth]{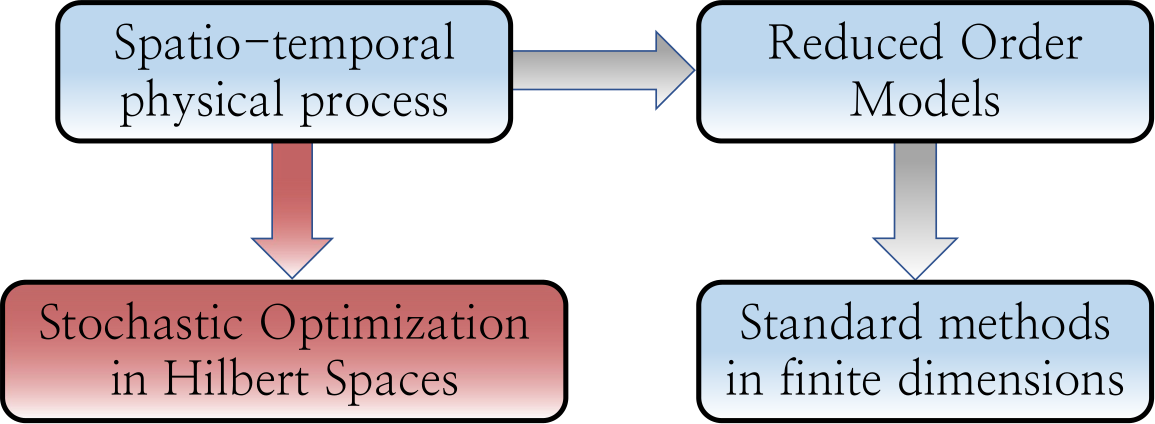}
  \caption{Our proposed approach versus traditional approaches.}
\label{fig:approach_diagram}
\end{figure}

The majority of recent results are composed of a growing body of work that often rely on machine learning techniques, and seek control of \acp{PDE} by immediately reducing them to a set of ODEs~\cite{bieker2019deep,nair2019cluster,mohan2018deep,morton2018deep,rabault2019artificial,satheeshbabu2019open}. They do not consider stochasticity and typically use standard tools from finite-dimensional control theory. In some cases, such as in \cite{morton2018deep}, the resulting methods can violate stabilizability conditions. In~\cite{rabault2019artificial} the authors successfully control a Navier-Stokes system with reinforcement learning on policy networks in a deterministic, finite \ac{ODE} setting. Similarly,~\cite{bieker2019deep} presents a Deep \ac{RNN} framework with \ac{MPC} to control a finite, deterministic \ac{ODE} representation (CFD solver) of a Navier-Stokes system. In the soft robotics setting, \cite{satheeshbabu2019open} applies deep reinforcement learning, more precisely deep Q-learning, on a discrete finite markov decision process representation of a soft pneumatic-driven manipulator in order to obtain an open-loop position controller to control deflections at the tip. In~\cite{spielberg2019learning}, the authors similarly apply standard finite dimensional deep learning methods for policy and actuator co-design optimization of deformable body robots for locomotion by wrapping clustering and deep reinforcement learning around a differentiable simulator. Other recent finite-dimensional machine learning-based methods are covered in the review paper \cite{george2018control}.

These finite dimensional machine learning methods generally rely on standard optimality principles from the finite dimensional \ac{SOC} literature, namely the Dynamic Programming (or Bellman) principle and the stochastic Pontryagin Maximum principle~\cite{Pontryagin1962,Bellman1964,yong1999stochastic}, which typically provide solutions to the \ac{HJB} equation that suffer from the curse of dimensionality. In contrast to typical Pontryagin and \ac{HJB} methods, the SDE control literature presents probabilistic representations of the \ac{HJB} \ac{PDE} that can solve scalability via sampling techniques~\cite{Pardoux_Book2014,Fleming2006} including iterative sampling and/or parallelizable implementations~\cite{williams2017model}.

In the context of actuator co-design optimization for \acp{PDE}, several works have addressed optimal placement of actuators and sensors in the linear regime. In \cite{yang2017optimal}, minimum-norm control methods are used to place actuators for the stochastic heat equation. Similarly, $H_\infty$ and $H_2$ objectives are used for placement of actuators in flexible structures in~\cite{lim1992method,nestorovic2013optimal,kasinathan2013h}, and for the linearized Ginzburg-Landau equation in~\cite{chen2011h,manohar2018optimal}. Other methods leverage properties inherent to linear systems, such as symmetry properties in linear \acp{PDE} in~\cite{grigoriev1997pinning}, linear system Gramians, as in~\cite{sinha2013optimal,vaidya2012actuator}, and level set methods based on Gramians that promise scalability in~\cite{amstutz2006new}. Aside from these methods which are appealing, yet constrained to linear systems, optimal actuator and sensor placement for stabilization of the nonlinear Kuramoto-Sivashinsky equation is demonstrated in~\cite{lou2003optimal}. They produce appealing results, however they impose strong simplifying assumptions which limit their dimensionality. Finally, conditions for the existence of optimal actuator and sensor placement for semilinear PDEs are obtained in~\cite{edalatzadeh2019optimal}.

This work builds the proposed framework on~\cite{Evans2019IDVRL}, wherein the authors create a semi-model-free episodic optimization framework for policy-based control of \acp{SPDE}. In contrast to recent work which first require developing deterministic \acp{ROM} and then using standard approaches from \ac{RL} or \ac{MPC}, the proposed approach treats the \ac{SPDE} system directly in Hilbert spaces, and derives a variational optimization framework for episodic reinforcement of policy networks as highlighted in red in \cref{fig:approach_diagram}.

\section{Mathematical Preliminaries} \label{sec:preliminaries}

\begin{table*}[t!]
    \captionof{table}{\label{tab:semilinear_pdes} Examples of commonly known semi-linear PDEs in a \textit{fields representation} with subscript $x$ representing partial derivative with respect to spatial dimensions and subscript $t$ representing partial derivatives with respect to time. The associated linear operators  $\A$ in the Hilbert space formulation are colored blue, while the associated nonlinear operators $F(t,X)$ in the Hilbert space formulation are colored violet.}
    \centering
    \begin{tabular}{  l  l  l }
        \textbf{Equation Name} & \textbf{Partial Differential Equation}  & \textbf{Field State} \\ 
        \hline\noalign{\smallskip}
        Heat & $u_t = {\color{blue}\epsilon u_{xx}}$ & Heat/temperature\\  
        Burgers (viscous) & $u_t = {\color{blue}\epsilon u_{xx}} - {\color{violet}u u_x}$ & Velocity\\
        Nagumo & $u_t = {\color{blue}\epsilon u_{xx}} + {\color{violet} u(1-u)(u-\alpha)}$ & Voltage\\
        Allen-Cahn & $u_t = {\color{blue}\epsilon u_{xx} } + {\color{blue}u} - {\color{violet} u^3}$ & Phase of a material \\
        Navier-Stokes & $u_t = {\color{blue}\epsilon \Delta u} - \nabla p - {\color{violet}(u \cdot \nabla)u}$ & Velocity\\ 
        Nonlinear Schrodinger & $u_t = {\color{blue}\frac{1}{2} i u_{xx}} + {\color{violet}i |u|^2 u=0}$ &  Wavefunction\\
        Korteweg-de Vries & $u_t = - {\color{blue} u_{xxxx}} - {\color{violet} 6uu_x} $ & Plasma wave \\
        Kuramoto-Sivashinsky & $u_t =  -{\color{blue}  u_{xx}}- {\color{blue}  u_{xxxx}} -{\color{violet}uu_x}$ & Flame front \\
    \end{tabular}
\end{table*}

The proposed approach applies principles of measure-theoretic optimization for policy and actuator co-design optimization for a large class of stochastic spatio-temporal systems represented as \acp{SPDE}. These systems can be conveniently described as evolving on time separable Hilbert spaces, where they are represented by infinite dimensional vectors and acted on by operators. Consider the class of \acp{SPDE} that are of \textit{semi-linear} form. Similar mathematical preliminaries are treated in our prior work \cite{evans2020spatio}. Let $\calH$ denote a separable Hilbert space with inner product $\langle \cdot, \cdot \rangle$,  $\sigma$-field $\calB(\calH)$, and probability space $(\Omega, \calF, \Pb)$ with filtration $\calF_t$, $t\in [0,T]$. Consider the general semi-linear form of a controlled \ac{SPDE} on $\calH$ given by 
\begin{align}\label{eq:SPDEs_Control}
\rd X = \big( \A X   + &F(t, X)  \big) \rd t + G(t, X)\big(\vPhi(t,X,\vx; \vTheta^{(k)})\rd t+ \frac{1}{\sqrt{\rho}} \rd  W(t)\big), 
\end{align}
where $X(t) \in \calH$ is the state of the system which evolves on the Hilbert space $\calH$, the linear and nonlinear measurable operators $\A: \calH \rightarrow \calH$ and $F(t,X): \Rb \times \calH \rightarrow \calH$  (resp.) are uncontrolled drift terms, $\vPhi(t,X,\vx; \vTheta^{(k)}):\Rb \times \calH \times \calD \rightarrow \calH$ is the nonlinear control policy parameterized by $\vTheta^{(k)}$ at the $k^{th}$ iteration, where $\calD \subset \Rb^3$ is the domain of the finite spatial region, $\rd W(t):\Rb \rightarrow \calH$  is a spatio-temporal noise process, and $G(t,X)$ is a nonlinearity that affects both the noise and the control and acts to incorporate each into the field. 

Despite appearing to limit the applicability of this approach, the semilinear classification covers a vast majority of natural systems governed by \acp{SPDE}. Several examples of semilinear \acp{PDE} can be found in \cref{tab:semilinear_pdes}. Despite being nonlinear from a dynamics perspective, each of these are classified as semilinear \acp{PDE}, which are simply \acp{PDE} which can be nonlinear in any order derivative of the state except the highest order derivative, which they must have a linear relationship with \cite{lototsky2017stochastic}. Linearity in the highest order derivative allows one to form a contractive, unitary, linear semigroup which is often used to guarantee existence and uniqueness of $\calF_t$-adapted weak solutions $X(t)$, $t \geq 0$. The interested reader can refer to \cite[Chapter 1]{fabbri} or \cite[Part II]{da2014stochastic} for a more complete description.

The \textit{Hilbert spaces} formulation given in \cref{eq:SPDEs_Control} is general in that \textit{any} semi-linear \ac{SPDE} can be described in this form by appropriately choosing the linear $\A$ and nonlinear $F$ operators. In this form, the spatio-temoral noise process $\rd W(t)$ is a Hilbert space-valued Wiener process, which is a generalization of the Wiener process in finite dimensions. We include a formal definition of a Wiener process in Hilbert spaces for clarity \cite[Section 4.1.1]{da2014stochastic}
\begin{definition} A $\calH$-valued stochastic process $W(t)$ with probability law $\calL\big(W(\cdot)\big)$ is called a Wiener process if
\begin{enumerate}[nosep]
    \item $W(0) = 0$
    \item $W$ has continuous trajectories
    \item $W$ has independent increments
    \item $\calL\big(W(t) - W(s)\big) = \calN (0, t-s)Q), \quad t \geq s \geq 0$
    \item $\calL\big(W(t)\big) = \calL\big(-W(t)\big), \quad t \geq 0$
\end{enumerate}
\end{definition}  \vspace{0.25cm}

\begin{proposition}
Let $ \{e_{i}\}_{i=1}^{\infty} $  be a complete orthonormal system for the Hilbert Space $\calH$. Let $Q$ denote the covariance operator of the Wiener process $W(t)$. Note that $Q$ satisfies $Q e_{i}= \lambda_{i} e_{i}$, where $\lambda_{i}$ is the eigenvalue of $Q$ that corresponds to eigenvector $e_{i}$. Then, $W(t) \in \calH$ has the following expansion:
 \begin{equation}\label{eq:Wiener_expansion}
    W(t) = \sum_{j=1}^{\infty} \sqrt{\lambda_{j}} \beta_{j}(t) e_{j},
\end{equation}
\noindent where  $ \beta_{j}(t)  $  are real valued Brownian motions that are mutually independent on $ (\Omega, \calF, \mathbb{P})$.
\end{proposition}
The expansion in \cref{eq:Wiener_expansion} reveals how the Wiener process acts spatially for a given basis. There are various forms of the Wiener process with different properties. We refer the interested reader to~\cite{da2014stochastic} for a more complete introduction. The proposed approach is derived for a special case of Wiener process called the Cylindrical Wiener process, defined as follows.
\begin{definition}
A Wiener process $W(t)$ on $\calH$ is called a Cylindrical Wiener process if the covariance operator $Q$ is the identity operator $I$.
\end{definition}

Note that for the Cylindrical Wiener process, the sum in \cref{eq:Wiener_expansion} is unbounded in $\calH$ since $\lambda_j = 1$, $\forall j=1,2,\dots$. 
This makes the Cylindrical Wiener process a challenging Wiener process to handle since it acts spatially \textit{everywhere} with \textit{equal} magnitude, in contrast to Wiener processes with covariance operators that are of trace class (i.e. wherein the expansion in \cref{eq:Wiener_expansion} is finite). This type of spatio-temporal noise requires the assumption that the operators $\A$, $F(t, X)$, and $G(t,X)$ satisfy properly formulated conditions given in \cite[Hypothesis 7.2]{da2014stochastic} to guarantee the existence and uniqueness of the $\calF_t$-adapted weak solution $X(t)$, $t \geq 0$.

\section{Second Order Soft-Robotic SPDEs in Direct Product Hilbert Spaces} \label{sec:second_order}
Many complex spatio-temporal systems are given by stochastic partial differential equations of second order in time. Second order \acp{SPDE} typically have behavior analogous to second order mechanical \ac{SDE} systems derived from Newtonian mechanics, in which the actuation acts as external forces or torques which enter through the derivative of the respective linear and rotational momenta. These are typically treated by defining a new momentum state, and writing the system in matrix-vector form. If one takes a robot arm constrained to one dimension, defined by a scalar second order \ac{SDE}, and repeatedly adds joints, and thereby degrees of freedom, one would obtain a one-dimensional continuum robot manipulator in the limit, which has infinite degrees of freedom and is described by a suitable one-dimensional second-order \ac{SPDE}. 

\subsection{The Euler-Bernoulli Continuum System}

One such system description is achieved in the simply supported stochastic Euler-Bernoulli equation with Kelvin-Voigt and viscous damping, which is a simplified model of a soft robotic limb. The Euler-Bernoulli equation is used extensively in beam theory, and has applications to a variety of robotic systems beyond soft robotics. Formally, the \ac{1D} Euler-Bernoulli equation with Kelvin-Voigt and viscous damping is given in \textit{fields representation} by
\begin{equation} \label{eq:EB_SPDE}
\begin{split}
&\partial_{tt} y + \partial_{xx} \big( \partial_{xx}y + C_d \partial_{xxt} y \big) + \mu \partial_t y = \vPhi + \frac{1}{\sqrt{\rho}} \partial_t W(t,x), \\
&y(t,0) = y(t,a) = 0, \\
&y(0,x) = y_0, \\
&\partial_t y(0,x) = v_0 , \\
&\partial_{xx}y(t,0) + C_d \partial_{xxt}y(t,0) = 0, \\
&\partial_{xx}y(t,a) + C_d \partial_{xxt} y(t,a) = 0,
\end{split}
\end{equation}
where $y(t,x)=y: D \times \Rb \rightarrow \Rb$ represents the vertical displacement of the beam, $C_d$ represents the Kelvin-Voigt damping coefficient, $\mu$ represents the viscous damping coefficient, and all functional dependencies of the nonlinear policy $\vPhi$ have been dropped since it has a different form in the fields representation. Note that the policy and stochastic effects enter as forces. With the change of variables $v := \partial_t y $, this system has the typical second order matrix-vector form
\begin{align} \label{eq:EB_matrix_form}
     \partial_t \left[\begin{array}{c} y \\ v \end{array} \right] &= \left[\begin{array}{cc} 0 & 1 \\ -A_0 & -C_d A_0 - \mu \end{array} \right] \left[ \begin{array}{c} y \\ v \end{array} \right ] + \left[ \begin{array}{c} 0 \\ 1 \end{array} \right] \Phi + \left[ \begin{array}{c} 0 \\ \frac{1}{\sqrt{\rho}} \end{array} \right] \partial_t W(t,x),
\end{align}
where $A_0 = \partial_{xxxx}$ without boundary conditions. Now, we lift this \ac{SPDE} into infinite dimensional Hilbert spaces. Define $Y \in \calH$ as the Hilbert space analog of $y(t,x)$, $V \in \calH$ as the Hilbert space analog of $v(t,x)$, and a variable $Z$ on the direct product Hilbert space $\calH^2 := \calH \times \calH$. Note that $Z$ is a Hilbert space analog of a variable $z(t,x) = [y(t,x) \;\; v(t,x)]^\top \in \Rb^2$. In Hilbert spaces, $A_0$ becomes an operator acting on $\calH$ and $1$ gets replaced by the identity operator $I$ acting on $\calH$. Rewriting \cref{eq:EB_matrix_form} in Hilbert space semi-linear form yields
\begin{equation} \label{eq:EB_Hilbert_form}
    \rd Z = \A Z \rd t + G\Big( \vPhi(t,Z,\vx; \vTheta^{(k)})\rd t + \frac{1}{\sqrt{\rho}}\rd W(t) \Big),
\end{equation}
where $\A: \calH^2 \rightarrow \calH^2$ is the linear operator $\A = [0 \;\; I; -A_0 \;\; -C_d A_0 - \mu I]$, $G: \calH \rightarrow \calH^2$ is an operator representing how control and spatio-temporal noise enter the system $G = [0; I]$, and $\rd W(t)$ is a Cylindrical Wiener process on $\calH$. Note that the Hilbert space variables $Y$, $V$, and $Z$ no longer have spatial dependence as the Hilbert space vectors capture the spatial continuum over which the problem is defined.


\subsection{Detailed Models of Soft-Robotic Limbs}

While the Euler-Bernoulli \ac{SPDE} has wide applicability, it relies on a small-angle assumption, which is not suitable for some soft-robotic applications such as in soft-robotic manipulators or end-effectors. In \cite{shapiro2015modeling} the errors introduced by violation of this approximation have been reduced significantly by parametrizing the beam's backbone by a tangent angle and including a single-parameter hysteretic term. The resulting model is used for a hyper-flexible system.

However, the majority of modern soft-body robotics modeling research typically deviates from the Euler-Bernoulli approach. The majority of modern models of soft-robotic systems are divided into two main categories: constant curvature approximation, and non-constant curvature approximation \cite{renda2014dynamic}. For a detailed review of constant and piecewise constant curvature methods as of 2010, refer to \cite{webster2010design}. More recent constant curvature methods include \cite{rone2013continuum}, wherein the authors use the principle of virtual power to derive a model with constant curvature segments and discrete torsional joints. In \cite{godage2016dynamics}, the authors start with a piecewise-constant curvature approximation, and produce a model that they then validate against a piece-wise constant curvature robotic manipulator. In \cite{falkenhahn2015dynamic}, the authors apply a peicewise constant curvature model to a continuum robotic manipulator actuated by pressure differentials provided by bellows. Constant and piecewise constant curvature models are often much simpler in implementation, yet these models can fail when the system does not have a uniform shape or is acted on by a large external load. The interested reader can refer to \cite{george2018control} for a recent review of control methodologies on constant and piecewise-constant curvature models of continuum manipulators.

Non-constant curvature models are increasing in popularity due to a typically more accurate representation of a continuum system. They are typically broken into three subcategories: continuum approximations of hyperredundant models, spring-mass models, and cosserat or geometrically exact models. The continuum approximation of hyperredundant models were among the first proposed continuum methods \cite{chirikjian1994hyper}, and led to several interesting applications \cite{hannan2003kinematics,mochiyama2005hyper,chirikjian1995kinematics}. On the other end of the spectrum, Cosserat models are currently the most exact models of continuum systems. Derived from the context of beam theory, these geometrically exact models often have a large number of \ac{PDE} states, making them quite difficult to simulate at high frequency. Yet, their high fidelity has made them an appealing research direction. Cosserat models have also been simulated in real-time for very slender rods with uniform cross-section in \cite{till2015efficient,till2019real}. In \cite{renda2014dynamic}, the authors develop a geometrically exact \ac{3D} model on Lie groups of a tentacle-like tapered soft robot arm actuated by cables, resulting in a \ac{PDE} with $18$ states. Similar models are also developed and validated in \cite{trivedi2008geometrically}. In addition to introducing significant model complexity, Cosserat \acp{PDE} are also known to suffer from \textit{stiff} dynamics with respect to the Courant-Friedrichs-Lewy stability condition~\cite{till2019real}. For a complete review of design, fabrication and control strategies that sweep across the discipline of soft robotics, the interested reader should refer to \cite{rus2015design}.

In between these two extremes are the set of continuum spring-mass models. These models often emerge in models of biological systems, such as the appendages of the octopus vulgaris \cite{zheng2012dynamic,yekutieli2005dynamic,walker2005continuum}. Their features include accurate, volume preserving representations of muscular forces and lower \ac{PDE} state dimensionality compared to Cosserat models. In \cite{etzmuss2003deriving}, the authors derive a particle-based model that falls into the category of spring-mass models, and they establish a link between such particle systems and continuum mechanics. In this work we consider a stochastic variant of their model, actuated by an actuation function modeled after muscular behavior common to cephalopods \cite{walker2005continuum}.

The \ac{SPDE} governing the dynamics of a continuum elastic material is given by
\begin{equation} \label{eq:continuumPDE}
    \rho_m \partial_{tt} \vs = \text{div}(\sigma) + \vf_g + \vPhi + \frac{1}{\sqrt{\rho}} \partial_t W(t),
\end{equation}
where $\rho_m$ is the material density, $\vs$ is the deformed state, $\sigma$ is the stress tensor, $\vf_g$ is the force of gravity, and $\vPhi$ is the nonlinear policy.
The material state $\vs$ can be expressed as the sum of an initial rest state $\vr$ and deformation $\vd$, each of which are parameterized over \ac{2D} domain $\calD = X \times Y$.
\begin{equation} \label{eq:deformation}
    \vs(t,x,y)=\vr(x,y)+\vd(t,x,y).
\end{equation}

The total stress tensor $\sigma$ in \cref{eq:continuumPDE} is the sum of the elastic and viscous stresses,
\begin{equation} \label{eq:sum_stresses}
   \sigma=\sigma^\epsilon+\sigma^\nu.
\end{equation}
Assuming linear elasticity, the elastic stress tensor $\sigma^\epsilon$ may be related to the strain tensor $\epsilon$ via the the stiffness tensor $C$.
\begin{equation} \label{eq:linearstresstensor}
    (\sigma^\epsilon)_{ij}=(C)_{ijkl}(\epsilon)_{kl}
\end{equation}
where subscript of a paranthesis $(A)_{ijkl}$ denotes tensor element $i,j,k,l$ of $A$, and Einstein summation notation is utilized to perform tensor contractions. Strains, denoted $\epsilon$, within the material are determined by Green's strain tensor, where subscripts of $s$ indicate partial derivatives with respect to coordinates $x$ or $y$.
\begin{equation} \label{eq:greenstrain}
    \epsilon=
    \begin{pmatrix}
       \Vert \vs_x \Vert^2 -1 & \frac{1}{2} \langle\vs_x,\vs_y\rangle  \\ 
       \frac{1}{2}  \langle\vs_x,\vs_y\rangle  & \Vert \vs_x \Vert^2 -1
    \end{pmatrix}.
\end{equation}
 For isotropic materials, entries of the stiffness tensor $C$ are determined by tensile constant $k$ and shear modulus $\mu$
\begin{equation} \label{eq:stiffnesstensor}
    (C)_{iiii}=k, \quad (C)_{ijij}=(C)_{jiij}=\frac{1}{2}\mu, \quad i\neq j.
\end{equation}

Dissipative effects can be modelled by Kelvin-Voigt damping, which adds a viscous stress $\sigma^\nu$ proportional to the strain rate $\nu$.
\begin{align} 
    (\nu)_{ij} &= \frac{\rd (\epsilon)_{ij}}{\rd t} \label{eq:strainratetensor} \\
    (\sigma^\nu)_{ij} &= (D)_{ijkl}(\nu)_{kl} \label{eq:visousstress}
\end{align}
The damping tensor $D$ is proportional to the stiffness tensor $C$ by a retardation time constant $\tau$.
\begin{equation} \label{eq:kelvinvoigt}
    D=\tau C
\end{equation}

In this case, the resulting \ac{SPDE} can be again lifted into Hilbert spaces in a similar fashion as in \cref{eq:EB_SPDE}. Define displacement velocity $\vv := \partial_t \vs = \partial_t \vd$, and rewrite \cref{eq:continuumPDE} in typical second order matrix-vector form
\begin{align} \label{eq:SM_matrix_form}
     \partial_t \left[\begin{array}{c} \vd \\ \vv \end{array} \right] &= \left[\begin{array}{cc} 0 & 1 \\ 0 & 0 \end{array} \right] \left[ \begin{array}{c} \vd \\ \vv \end{array} \right]  + \left[ \begin{array}{c} 0 \\ \frac{\text{div}(\sigma) + \vf_g}{\rho_m} \end{array} \right]  + \left[ \begin{array}{c} 0 \\ \frac{1}{\rho_m} \end{array} \right] \Phi + \left[ \begin{array}{c} 0 \\ \frac{1}{\rho_m \sqrt{\rho}} \end{array} \right] \partial_t W(t,x).
\end{align}
We again lift this \ac{SPDE} into infinite dimensional Hilbert spaces. Note that $\vd = \vd(t,x,y)$ and $\vv=\vv(t,x,y)$ have $x$ and $y$ components that are each defined over \ac{2D} spatial domain $\calD= X\times Y$. Define $\calW \in \calH \times \calH$ as the Hilbert space analog of $\vd(t,x,y)$, $\calV \in \calH \times \calH$ the Hilbert space analog of $\vv(t,x,y)$, and a variable $Z$ on the direct product Hilbert space $\calH^4 = \calH \times \calH \times \calH \times \calH$. This new variable $Z$ is a Hilbert space analog of a variable $\vz(t,x,y) = [\vd(t,x,y)\;\; \vv(t,x,y)]^\top \in \Rb^4$. Rewriting \cref{eq:SM_matrix_form} in Hilbert space semi-linear form yields
\begin{equation} \label{eq:SM_Hilbert_form}
    \rd Z = \A Z \rd t + F(Z) \rd t +  G\Big( \vPhi(t,Z,\vx; \vTheta^{(k)})\rd t + \frac{1}{\sqrt{\rho}}\rd W(t) \Big)
\end{equation}
where $\A: \calH^4 \rightarrow \calH^4$ is a linear operator, $F$ is the nonlinear operator which contains the forces due to stresses and gravity, $G:\calH^2 \rightarrow \calH^4$ is an operator representing how the $\calH^2$-valued control policy and spatio-temporal noise enter the system, and $\rd W(t)$ is a Cylindrical Wiener process on $\calH^2$. Again, the Hilbert space variables lose spatial dependence as they represent the entire spatial continuum.


\section{Girsanov Theorem for Second Order SPDEs} \label{sec:girsanov}

The proposed approach is derived from the perspective of a measure theoretic view of variational optimization, wherein the change of measures, or \ac{RN} derivative is a tool that is widely leveraged to change the sampling distribution of an expectation. Thus, such a framework requires a properly formulated Girsanov theorem for second order \acp{SPDE} defined on time-indexed Hilbert spaces. This was first presented in our prior work \cite{evans2020spatio}, and is repeated here for clarity. 

\begin{theorem}[Girsanov] \label{girs} Let $\Omega$ be a sample space with a $\sigma$-algebra $\mathcal{F}$. Consider the following $\calH^2$-valued (or similarly $\calH^4$-valued) nonlinear stochastic processes
\begin{align}
\rd Z&=\big(\A Z+F(t, Z)\big) \rd t +  \frac{1}{\sqrt{\rho}}G(t, Z)\rd W(t), \label{eq:Z}\\
\rd\tilde{Z}&=\big(\A \tilde{Z}+F(t, \tilde{Z})\big)\rd t+G(t, \tilde{Z})\bigg(B(t, \tilde{Z})\rd t + \frac{1}{\sqrt{\rho}} \rd W(t)\bigg),\label{eq:Z_tilde}
\end{align}
where $Z(0)=\tilde{Z}(0)=z_0$ and $W \in \calH$ (or similarly $W \in \calH^2$) is a Cylindrical Wiener process with respect to measure $\mathbb{P}$. Moreover, let $\Gamma$ be a set of continuous-time, infinite-dimensional trajectories in the time interval $[0,T]$. Define the {\it probability law} of $Z$ over trajectories $\Gamma$ as $\mathcal{L}(\Gamma):=\mathbb{P}(\omega\in\Omega|Z(\cdot,\omega)\in\Gamma)$. Similarly, define the law of $\tilde{Z}$ as $\tilde{\calL}(\Gamma):=\mathbb{P}(\omega\in\Omega|\tilde{Z}(\cdot,\omega)\in\Gamma)$. Assume 
\begin{equation}
    \mathbb{E}_{\mathbb{P}}\big[e^{\frac{1}{2}\int_{0}^{T}||\psi(t)||^2\mathrm dt}\big]<+\infty,
\end{equation}
where
\begin{equation}
    \psi(t):=\sqrt{\rho}\tilde{B}\big(t, Z(t)\big)\in \calH.
\end{equation}
Then
\begin{equation} \label{eq:Lg}
\begin{split}
\tilde{\calL}(\Gamma) =  \mathbb{E}_{\mathbb{P}}\Bigg[ \exp\bigg(&\int_{0}^{T} \big\langle\psi(s),\rd W(s)\big\rangle-\frac{1}{2}\int_{0}^{T}\big|\big|\psi(s)\big|\big|^{2}\rd s\bigg) \Bigg| X(\cdot)\in\Gamma\Bigg],
\end{split}
\end{equation}

\end{theorem}
\begin{proof}
Define the process
\begin{equation}
\label{eq:w_hat}
\hat{W}(t):= W(t)-\int_{0}^{t}\psi(s)\rd s.
\end{equation}
Under the above assumption, $\hat{W}$ is a Cylindrical Wiener process with respect to a measure $\Qb$ defined by
\begin{equation}
\label{eq:girsanov_measure}
\begin{split}
\rd \Qb (\omega)&=\exp\bigg(\int_{0}^{T}\big\langle\psi(s),\rd W(s)\big\rangle-\frac{1}{2}\int_{0}^{T}\big|\big|\psi(s)\big|\big|^{2}\rd s\bigg)\rd\mathbb{P} \\ &=\exp\bigg(\int_{0}^{T}\big\langle\psi(s),\rd \hat{W}(s)\big\rangle+\frac{1}{2}\int_{0}^{T}\big|\big|\psi(s)\big|\big|^{2}\rd s\big)\rd\mathbb{P}.
\end{split}
\end{equation}
The proof that $\hat{W}$ is a Cylindrical Wiener process with respect to $\Qb$ can be found in \cite[Theorem 10.14]{da2014stochastic}. Now, using \cref{eq:w_hat}, \cref{eq:Z} is rewritten as
\begin{align}
\rd Z &= \big(\A  X+F(t, Z)\big)\rd t+\frac{1}{\sqrt{\rho}}G(t, Z)\rd W(t) \label{eq:Z_new0}  \\
     &=\big(\A  Z+F(t, Z)\big)\rd t+G(t, Z)\bigg(B(t,Z)\rd t+\frac{1}{\sqrt{\rho}}\rd\hat{W}(t) \bigg) \label{eq:Z_new1}
\end{align}
Notice that  the  SPDE in \cref{eq:Z_new1} has the same form as \cref{eq:Z_tilde}. Therefore, under the introduced measure $\Qb$ and noise profile $\hat{W}$, $Z(\cdot, \omega)$ becomes equivalent to $\tilde{Z}(\cdot, \omega)$. Conversely, under measure $\mathbb{P}$, \cref{eq:Z_new0} (or \cref{eq:Z_new1}) behaves as the original system in \cref{eq:Z}. In other words, \cref{eq:Z} and \cref{eq:Z_new1} describe the same system on $(\Omega, \mathcal{F}, \mathbb{P})$. From the uniqueness of solutions and the aforementioned reasoning, one has
\[\Pb\big(\{\tilde{Z}\in\Gamma\}\big) = \Qb\big(\{Z\in\Gamma\}\big).\]
The result follows from \cref{eq:girsanov_measure}. $\qed$
\end{proof}

The notion most pertinent to the subsequent derivation is the change of measures or \ac{RN} derivative between the associated measures of the uncontrolled and controlled systems defined in \cref{eq:Z} and \cref{eq:Z_tilde}, respectively, and is given by
\begin{equation} \label{eq:RN}
\begin{split}
\frac{\rd \calL}{\rd \tilde{\calL}} = \exp\bigg(-\int_{0}^{T}\big\langle\psi(s), \rd W(s)\big\rangle-\frac{1}{2}\int_{0}^{T}||\psi(s)||^{2}\rd s\bigg).
\end{split}
\end{equation}
In the case of semilinear \acp{SPDE} of the form \cref{eq:EB_Hilbert_form} and similarly any semilinear \acp{SPDE} in the general form \cref{eq:SPDEs_Control}, the function $\psi$ which defines this RN derivative is given by
\begin{equation}
    \psi(t):=\sqrt{\rho}\vPhi(t,Z, \vx; \vTheta^{(k)})
\end{equation}
which simplifies the RN derivative to
\begin{equation} \label{eq:RN_semi_linear}
\begin{split}
\frac{\rd \calL}{\rd\tilde{\calL}} = \exp\bigg(&-\sqrt{\rho}\int_{0}^{T}\big\langle\vPhi(t,Z,\vx; \vTheta^{(k)}), \rd W(s)\big\rangle -\frac{\rho}{2}\int_{0}^{T}\big|\big|\vPhi(t,Z,\vx; \vTheta^{(k)})\big|\big|^{2}\rd s\bigg).
\end{split}
\end{equation}
For convenience, we assign functions to each term in \cref{eq:RN_semi_linear}
\begin{align}
    \calN(\vTheta, \vx) &:= \int_{0}^{T}\big\langle\vPhi(t,Z,\vx; \vTheta^{(k)}), \rd W(s)\big\rangle \\
    \calP(\vTheta, \vx) &:= \int_{0}^{T}\big|\big|\vPhi(t,Z,\vx; \vTheta^{(k)})\big|\big|^{2}\rd s
\end{align}

\section{Spatio-Temporal Stochastic Optimization} \label{sec:stso}

The proposed measure theoretic framework was first derived in \cite{Evans2019IDVRL} for the simpler case of policy optimization without co-design optimization. In \cite{evans2020spatio} this work was extended to policy and actuator co-design optimization. These frameworks are explicit feedback formulation of the feedforward and \ac{MPC} formulations given in \cite{boutselis2019variational}. The explicit feedback is realized through the nonlinear policy $\vPhi(t,X,\vx; \vTheta^{(k)})$, which is a potentially time-varying policy that has explicit state dependence. 

Nonlinear, explicit state dependence allows for a feedback policy that can extract pertinent information from the state for control, and is in a sense reactive to undesired evolutions of the state. Policy networks have had widespread success in extracting pertinent features in a multitude of systems, and are utilized here for the nonlinear policy. Embedded in this function is also a dependence on $\vx$, which describes how the actuator may depend on some design variables, such as actuator placement in the spatial domain. This approach also encompasses cases where terms that parametrize how the actuators are shaped or sized are included in the nonlinear policy.

As discussed in \cref{sec:intro}, the proposed framework is based on an instantiation of the second law of thermodynamics given in \cref{eq:Free_Energy_Relative_Entropy} in the following form~\cite{theodorou2015nonlinear,theodorou2018linearly}
\begin{align} \label{eq:Legendre}
  & - \frac{1}{\rho}  \log \Eb_{\calL} \bigg[ \exp( -\rho {J} )  \bigg]  = \min_{\vTheta, \vx} \bigg[    \Eb_{\tilde{\calL}}\left({J} \right)  + \frac{1}{\rho} D_{KL} ( \tilde{\calL}\; \big|\big|  \calL )  \bigg],
\end{align}
where $J=J(X)$ is an arbitrary state cost functional. Relating \cref{eq:Legendre} to \cref{eq:Free_Energy_Relative_Entropy}, the metaphorical work and entropy describe a metaphorical energy landscape for which there is a minimizing measure. Sampling from this measure would simultaneously minimize state cost and the $KL$-divergence term, which is interpreted as control effort. The measure that optimizes \cref{eq:Legendre} is the so-called Gibbs measure
\begin{equation}\label{eq:Gibbs}
\rd \calL^{*} = \frac{\exp( - \rho J) \rd \calL}{\Eb_\calL \big[\exp( - \rho J) \big] }.
\end{equation}

The significance of \cref{eq:Legendre} from the perspective of optimal control theory lies in established connections between \cref{eq:Legendre} and the \ac{HJB} equation in infinite dimensions, as shown in~\cite{theodorou2018linearly}. This connection to a foundational principle in optimal control literature motivates the use of \cref{eq:Legendre} and the resulting optimal measure in \cref{eq:Gibbs} for the derivation of the proposed measure-theoretic optimization strategy.

It is not known how to sample directly from the Gibbs measure in \cref{eq:Gibbs}. Instead, variational optimization methods are often used to iteratively minimize the controlled distribution's "distance" \footnote{Distance here is defined in the Kullback-Liebler (KL) sense, and is abusive terminology since the KL-Divergence is non-symmetric, and therefore not a distance metric in the mathematical sense.} to the Gibbs measure~\cite{williams2016aggressive,theodorou2018linearly,boutselis2019variational}. Define the control policy and actuator co-design problem as
\begin{subequations}\label{eq:theta_and_x}
\begin{align}
        \vTheta^{*} &=  \argmin_{\vTheta}  D_{KL}(\calL^{*}|| \tilde{\calL}) \\
        \vx^{*} &= \argmin_{\vx} D_{KL}(\calL^{*}|| \tilde{\calL})
\end{align}
\end{subequations}

Typically, actuator co-design optimization with policy optimization is performed in outer and inner loops, respectively, where the policy optimization is performed more often than the actuator co-design optimization. However, throughout experiments, the authors found that a joint optimization problem dramatically outperforms the split problem in \cref{eq:theta_and_x}. To make this clear, define a new variable $\hat{\vTheta} := [\vTheta, \; \vx]^\top$, and with it the new joint variational optimization as
\begin{equation}
    \hat{\vTheta}^{*} = \argmin_{\hat{\vTheta}} D_{KL}(\calL^{*}|| \tilde{\calL}).
\end{equation}
Expanding the KL divergence and applying the chain rule yields
\begin{align}
    \hat{\vTheta}^{*} &= \argmin_{\hat{\vTheta}} \bigg[  \int \log \Big( \frac{ \rd \calL^*}{\rd \calL} \frac{\rd \calL}{\rd \tilde{\calL}} \Big)  \rd  \calL^* \bigg],
\end{align}
which is equivalent to minimizing
\begin{align}
    \hat{\vTheta}^{*} &= \argmin_{\hat{\vTheta}} \bigg[ \int \log \Big( \frac{\rd \calL}{\rd \tilde{\calL}} \Big) \rd  \calL^* \bigg].
\end{align}  
Performing importance sampling yields
\begin{align}\label{eq:min_theta_x}
    \hat{\vTheta}^{*} &= \argmin_{\hat{\vTheta}} \bigg[\int \log \Big(\frac{\rd \calL}{\rd \tilde{\calL}} \Big)  \frac{\rd  \calL^*}{\rd \calL} \frac{\rd  \calL}{\rd \tilde{\calL}} \rd \tilde{\calL} \bigg].
\end{align}
The proposed iterative approach performs episodic reinforcement with respect to a loss function in order to optimize \cref{eq:min_theta_x}. Define the loss function as
\begin{align}\label{eq:def_loss_function}
    L(\hat{\vTheta}) &:= \Eb_{\tilde{\calL}}  \Bigg[\log \Big(\frac{\rd \calL}{\rd \tilde{\calL}} \Big)  \frac{\rd  \calL^*}{\rd \calL} \frac{\rd  \calL}{\rd \tilde{\calL}} \Bigg]
\end{align}
Plugging \cref{eq:RN} and \cref{eq:Gibbs} into \cref{eq:def_loss_function} yields
\begin{align}\label{eq:loss_function}
    L(\hat{\vTheta}^{(k)}) =  \Eb_{\tilde{\calL}}  \vast[ \frac{\exp( - \rho \tilde{J})}{\Eb_{\tilde{\calL}} \big[ \exp( - \rho \tilde{J}) \big]} \bigg( &-\sqrt{\rho} \calN(\hat{\vTheta}^{(k)}) - \frac{\rho}{2} \calP(\hat{\vTheta}^{(k)}) \bigg) \vast],
\end{align}
where $\tilde{J}= \tilde{J}(Z_{0:T}, \hat{\vTheta}^{(k)})$ is defined as
\begin{equation}\label{eq:importance_sampled_cost}
    \tilde{J}(Z_{0:T}, \hat{\vTheta}^{(k)}) := J(Z_{0:T}) + \frac{1}{\sqrt{\rho}}\calN(\hat{\vTheta}^{(k)}) +\frac{1}{2}\calP(\hat{\vTheta}^{(k)}),
\end{equation}
and $J(Z_{0:T})$ is a state cost evaluated over the state trajectory $Z_{0:T}$. For reaching tasks, $J(Z_{0:T})$ is typically a weighted 2-norm distance to the goal state.

This loss function compares sampled trajectories by evaluating them on the exponentiated $\tilde{J}$ performance metric. The importance sampling terms $\calN$ and $\calP$, which appear in $\tilde{J}$ add a quadratic control penalization term and a mixed control noise term. In the Loss function, they serve as weights for the exponentiated cost trajectories. For convenience, we denote the exponentiated cost term as
\begin{equation} \label{eq:gibbs_final}
    \calE(Z_{0:T}, \vTheta^{(k)}) :=  \frac{\exp\big( - \rho \tilde{J}(Z_{0:T}, \vTheta^{(k)})\big)}{\Eb_{\tilde{\calL}} \Big[ \exp\big( - \rho \tilde{J}(Z_{0:T}, \vTheta^{(k)})\big) \Big]}
\end{equation}

Recall, that the nonlinear policy $\vPhi$ is a functional mapping into Hilbert space $\calH$ (or $\calH^2$). This is kept general for derivation purposes, however it implies that the nonlinear policy controls each element of an infinite vector in Hilbert space $\calH$ (or $\calH^2$). A more realistic, but less general representation refines the policy as 
\begin{equation}\label{eq:finite_actuation}
    \vPhi(t,Z,\vx;\vTheta^{(k)}) = \vm(\vx)^\top \varphi(Z ; \vTheta^{(k)} ),
\end{equation}
where $\vm(\vx):\calD^N \rightarrow \Rb^N \times \calH$ represents the effect of the actuation from $N$ actuators on the infinite-dimensional field. Typically this is either a Gaussian-like exponential with mean centered at the actuator locations or an indicator function. 

In \cref{eq:finite_actuation}, $\varphi(X; \vTheta^{(k)}): \calH \rightarrow \Rb^N$ is a policy network with $N$ control outputs representing $N$ distributed (or boundary) actuators. Note that as desired, the tensor contraction given on the right hand side of \cref{eq:finite_actuation} produces a vector in $\calH$ (or $\calH^2$). Splitting the actuation function from the control signal is also desired because we ultimately wish to use a finite input, finite output policy network for the function $\varphi(X; \vTheta^{(k)})$. The inner product terms become
\begin{align}
    \calN(\hat{\vTheta}^{(k)}) &= \int_{0}^{T}\big\langle\vm(\vx)^\top \varphi(Z ; \vTheta^{(k)} ), \rd W(s)\big\rangle \label{eq:noise_inner_product} \\
    \calP(\hat{\vTheta}^{(k)}) &= \int_{0}^{T}||\vm(\vx)^\top \varphi(Z ; \vTheta^{(k)} )||^{2}\rd s \nonumber \\
    &= \int_{0}^{T}\big\langle \varphi(Z ; \vTheta^{(k)} ), \vM(\vx)\varphi(Z ; \vTheta^{(k)} ) \big\rangle \rd s \label{eq:policy_inner_product}
\end{align}
where $\vM(\vx) = \vm(\vx) \vm(\vx)^\top$.

Many state of the art methods for training networks rely on a gradient approach~\cite{duchi2011adaptive,kingma2014adam}, wherein one computes a loss function that depends on the network parameters, and iteratively updates the network parameters based on the gradients of the loss with respect to said network parameters. Bootstrapping off the wide-spread success of these methods, we prescribe a similar gradient-descent update, which can be interchanged with any such gradient approach, given by
\begin{align}
    \vTheta^{(k+1)} &= \vTheta^{(k)} - \gamma_\vTheta \nabla_\vTheta L(\vTheta^{(k)}, x^{(k)}) \label{eq:policy_update}\\
    \vx^{(k+1)} &= \vx^{(k)} - \gamma_\vx \nabla_\vx L(\vTheta^{(k)}, \vx^{(k)}) \label{eq:actuator_update}
\end{align}
where $\gamma_\vTheta$ and $\gamma_\vx$ are the learning rates for the policy parameters and actuator design parameters, respectively, and $\nabla_a := \frac{\partial}{\partial a}$, denotes the partial derivative with respect to some finite-dimensional vector $a$.

Figure \ref{fig:alg_diagram} is a graphical representation of our approach. A Hilbert space policy network with initialized weights is passed through an SPDE model or physical realization of the system to produce state trajectories, which are used to compute a state cost as well as a sparse tensor that is used to compute the inner products in a memory and time-efficient manner. This method will be explained further in the subsequent section. The loss is computed and separate gradients are computed for the policy and actuator design. These gradients are used in conjunction with a gradient-based optimization algorithm such as \ac{SGD} to provide parameter updates to the policy network and actuator design. This approach is independent of discretization scheme, choice of actuation design components, choice of state cost functional, and choice of policy network.

\begin{figure}[ht!]
\centering
  \includegraphics[width=1.0\linewidth]{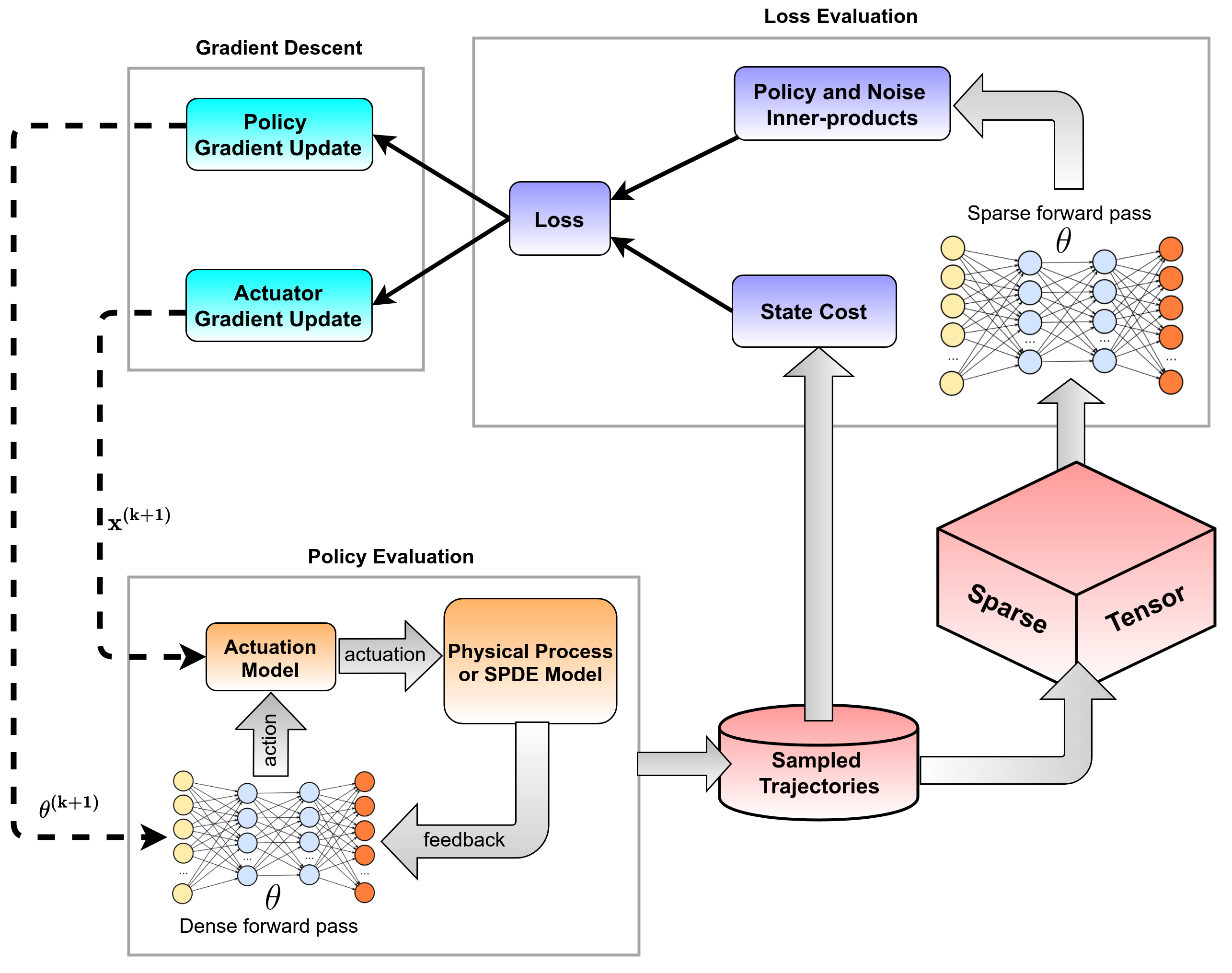}
  \caption{Diagram of the spatio-temporal stochastic optimization (STSO) approach for policy and actuator co-design optimization.}
\label{fig:alg_diagram}
\end{figure}

A key observation is that up to this point, we have a continuous-time optimization approach defined completely in Hilbert spaces; we have not performed any discretization of time or space. The benefit of this fact is that it equips our framework with the property of being \textit{discretization agnostic}. In other words, \textit{any} discretization scheme, for \textit{any} semi-linear \ac{SPDE} with additive Cylindrical Wiener process can be used in conjunction with the proposed algorithm. In fact, since the only term from the dynamics to appear in \cref{eq:def_loss_function} and \cref{eq:importance_sampled_cost} is the Cylindrical Wiener process $\rd W$, the optimization approach may consider the system and actuation model as a differentiable \textit{black-box}; one needs only the model of the additive Cylindrical Wiener process. In what follows, we consider \textit{any} discretization of the system, and provide numerical methods to handle difficulties that arise after discretization.

\section{Discrete Approximation Methods} \label{sec:approx}

The above derivation provides a general Hilbert space framework for optimizing nonlinear policies to control \acp{SPDE} to achieve some task. This approach represents an \textit{optimize-then-discretize} scheme. In order to implement the approach as an algorithm on a digital computer, data must be collected at finite frequency from interactions with a real system, or generated by a discretized physics-based or data-based model. In this section, we address the implement-ability of our approach with these considerations in mind.


\subsection{Sparse Spatial Integration}

Unique to this approach for training policy networks are the inner products that appear in \cref{eq:noise_inner_product} and \cref{eq:policy_inner_product}. Each of these Hilbert space inner products represent a spatial integration over the finite region $\calD$. Numerical methods to efficiently compute these spatial inner products were first developed in \cite{Evans2019IDVRL}. Consider the inner product in \cref{eq:policy_inner_product}. For 2D systems, it can be represented as a spatial integral in the form
\begin{align} \label{eq:inner_product_computation}
    &\int_{0}^{T}\Langle \varphi(Z(t) ; \vTheta^{(k)} ), \vM(\vx)\varphi(Z(t) ; \vTheta^{(k)} ) \Rangle \rd t \\
    &= \int_{0}^{T} \iint_D \varphi(Z(t,x,y) ; \vTheta^{(k)} )^\top \vM(x,y) \varphi(Z(t,x,y) ; \vTheta^{(k)} ) \rd x\, \rd y\,   \rd t \nonumber \\
    &= \int_{0}^{T} \sum_{j=1}^{\infty} \varphi(Z(t, e_j) ; \vTheta^{(k)} )^\top \vM(e_j) \varphi(Z(t, e_j) ; \vTheta^{(k)} ) \rd t,
\end{align}
where $\lbrace e_j \in \calH : j = 0,1,2,\dots \rbrace$ forms an orthonormal basis over $\calH$. After applying a spatial discretization to the system, the basis becomes a finite set $\lbrace e_j \in \Rb^{J^2} : j = 0,1,2,\dots \rbrace$, where $J$ is the number of discretization points in each dimension \footnote{We assume without loss of generality, that each dimension has the name number of discretization points $J$.}. One choice of such a basis is the set of one-hot vectors which emerges naturally from applying a central difference discretization, however, one may use a different basis or project onto the one-hot basis. Therefore, this integration scheme is also agnostic to the choice of discretization. Thus, evaluating the spatial integral is reduced to summing up forward passes through the policy network with each pixel considered individually.

Motivated by this one-hot basis approach, in \cite{Evans2019IDVRL} we developed a sparse matrix method for efficiently handling the spatial integrals, which become integrals of time-indexed ($J^2, J, J$) tensors for each sample. The key observation is that since the basis elements of each $(J,J)$ ``image" have only one activated ``pixel", the resulting tensor is tremendously sparse. As such, each layer's activation can be computed with a sparse matrix multiplication, resulting in the so-called $SparseForwardPass$ method that is not memory intensive for relatively large 2D problems. This can be applied to policy network architectures that utilize fully connected layers and convolutional layers. For convolutional input layers, this can be achieved by representing the convolution as a matrix multiplication with a Toeplitz-like matrix constructed from the filter coefficients~\cite{chellapilla2006high}.

\subsection{Approximate Discrete Optimization}

On the side of actuator co-design optimization, it is useful to refine the optimization procedure in \cref{eq:policy_update} and \eqref{eq:actuator_update} for certain optimization variables in light of the discretization. Such a case would be the placement of actuators, where depending on the actuation function, the system may not feel the effect of an actuator placed between discretization points.

To see this more clearly, consider the \ac{1D} spatial continuum $\calD= [0,1]$ discretized into a 11 point \ac{1D} grid. Lets assume that an actuator is chosen to be placed at $x=0.25$. Even though the actuation function $\vm(\vx)$ may be Gaussian-like function, the majority of the actuation will be felt in between two grid points, namely 0.2 and 0.3. This problem is even more severe if the actuation function $\vm(\vx)$ is the indicator function, as there will be \textit{no} actuation exerted on the field \textit{irrespective} of the control signal magnitude. Denote the number of spatial discretization points as $J$ and a \ac{3D} discretized problem domain grid as $\hat{\calD}$ composed of $J^3$ elements. Let $\vx_p$ denote the subset of optimization variables of $\vx$ that capture the placement of actuators, and $\vx_c$ as the rest of the elements of $\vx$. The optimization problem becomes
\begin{equation} \label{eq:discrete_min_theta_hat}
    \begin{split}
        \min_{\vTheta, \vx}\;\; &L(\vTheta, \vx) \\
        \text{subject to}\;\; &\vx_p \in \hat{\calD}
    \end{split}
\end{equation}

This formulation is an accurate representation, yet limits gradient flow from the loss function back to the actuator design parameters. In order to maintain these gradients, \cite{evans2020spatio} develops the following approximate approach. Define a one-to-one map $S:\hat{\calD} \rightarrow \Zb_+$, where $\Zb_+$ denotes positive integers. Applying the forward and inverse mapping produces a gradient-based parameter update of the form 
\begin{align}
    \vTheta^{(k+1)} &=  \vTheta^{(k)} - \gamma_\vTheta \nabla_{\vTheta}L(\vTheta^{(k)}, \vx_p^{(k)}, \vx_c^{(k)}) \label{eq:theta_update}\\
    \vx_c^{(k+1)} &= \vx_c^{(k)} - \gamma_{\vx_c} \nabla_{\vx_c} L(\vTheta^{(k)}, \vx_p^{(k)}, \vx_c^{(k)}) \label{eq:xc_update}\\
    \vx_p^{(k+1)} &= S^{-1}\bigg(R\Big(S\big(\vx_p^{(k)} - \gamma_{\vx_p} \nabla_{\vx_p} L(\vTheta^{(k)}, \vx_p^{(k)}, \vx_c^{(k)})\big)\Big)\bigg) \label{eq:xp_update}
\end{align}
where $R(\cdot)$ simply rounds to the nearest integer. This approach allows the use of well-known gradient update algorithms such as ADA-Grad~\cite{duchi2011adaptive} and ADAM~\cite{kingma2014adam}. See \cite{ruder2016overview} for an overview of popular gradient update algorithms used in machine learning.

\subsection{Modified Virtual Approximate Discrete Optimization} \label{subsec:modified_opt}

There is a key limitation with the above approach. In the case of a small gradient, the effect of rounding may "override" the effect of the gradient update. Thus, the gradient update may be prevented from changing the value until the gradient is large enough to push the variable close to the next discretization point. This effect becomes especially pronounced in the case of a course discretization, but also becomes apparent when the actuator placement is near an optimal value. In this local region, the gradient is relatively flat, so improper tuning of the learn rate combined with a course discretization grid would result in convergence to a sub-optimal value.

In this work we propose the following novel modification. Consider a virtual optimization variable $\vv \in \calD$ to serve as an intermediary in the optimization process. The goal of this intermediary is to preserve the gradient movement from the update, yet only allow the true optimization variable $\vx_p \in \hat{\calD}$ to exist on the discretization grid. Instead of applying the $S^{-1}\Big(R\big(S(\cdot)\big)\Big)$ map to the same variable update as in \cref{eq:xp_update}, we wish to carry the true gradient update information over iterations, so that the effect of the iterative update is additive over iterations. However, the issue is that the map $S^{-1}\Big(R\big(S(\cdot)\big)\Big)$ is a non-differentiable map due to the rounding in $R(\cdot)$. The proposed solution is to modify the optimization problem as follows
\begin{align}
    \vTheta^{(k+1)} &=  \vTheta^{(k)} - \gamma_\vTheta \nabla_{\vTheta}L(\vTheta^{(k)}, \vx^{(k)}) \label{eq:mod_theta_update}\\
    \vx_c^{(k+1)} &= \vx_c^{(k)} - \gamma_{\vx_c} \nabla_{\vx_c} L(\vTheta^{(k)}, \vx_p^{(k)}, \vx_c^{(k)}) \label{eq:mod_xc_update}\\
    \vv^{(k+1)} &= \vv^{(k)} - \gamma_{\vx_p} \nabla_{\vx_p} L(\vTheta^{(k)}, \vx_p^{(k)}, \vx_c^{(k)}) \label{eq:mod_v_update}\\
    \vx_p^{(k+1)} &= S^{-1}\Big(R\big(S( \vv^{(k+1)} ) \big) \Big) \label{eq:mod_xp_update}
\end{align}

Here, we carry two variables: a continuous-valued  variable $\vv \in \calD$, and a discrete-valued variable $\vx_p \in \hat{\calD}$. We compute the gradient of the loss $L$ with respect to the \textit{discrete}-valued variable $\nabla_{\vx_p} L(\vTheta^{(k)}, \vx_p^{(k)}, \vx_c^{(k)})$, but apply this gradient in a gradient update to the \textit{continuous}-valued variable $\vv$, which in effect bypasses the non-differentiable map. This important difference results in a virtual optimization variable $\vv$ which can reach the true optimal value, but is not applied to the system, and a secondary variable $\vx_p$ which represents the grid element nearest to the optimal value, and is applied to the system.


\section{Algorithm and Network Architecture} \label{sec:alg}

As discussed previously, implementation of the above framework requires spatial and temporal discretization of the \acp{SPDE} discussed in \cref{sec:preliminaries}. With this in mind, we choose an \ac{ANN} for our nonlinear policy $\varphi(Z; \vTheta^{(k)})$. In this work we use  \acp{FNN} for 1D experiments, and \acp{CNN} for 2D experiments. We use physics-based models of each \ac{SPDE} to generate training data. Given that the proposed framework is semi-model-free, real system data can seamlessly replace the physics-based model as described in~\cite{Evans2019IDVRL}. We only need prior knowledge of the flavor of noise, a differentiable model\footnote{Note that the actuation model can also be a black-box model} of the actuation function $\vm(\vx)$, and the actuator design elements $\vx$.

The resulting modified algorithm, modified from the original algorithm named \ac{ADPL} in \cite{evans2020spatio}, is referred to here as \ac{STSO} and shown in \cref{Algorithm1}. Here we modify the notation as well to specify the role of rollouts by using superscript $r$ to denote rollout $r \in R$, and superscript $0\!:\!R$ to denote the collected set of rollouts. We also generalize to optimizing over actuator placement and other non-placement actuator design variables, such as actuator variance. The inputs can change depending on the specific problem but in most cases contain time horizon ($T$), number of iterations ($K$), number of rollouts ($R$), initial state ($Z_0$), number of actuators ($N$), noise variance ($\rho$), time discretization ($\Delta t$), initial actuator variance ($\sigma_\mu^{(0)}$), initial network parameters ($\vTheta^{(0)}$), initial actuator locations ($\vx_p^{(0)}$), policy learn rate ($\gamma_\vTheta$), actuator location learn rate ($\gamma_{\vx_p}$), and actuator shape learn rate ($\gamma_{\vx_c}$). For more information on $SampleNoise()$, refer to~\cite[Chapter 10]{lord_powell_shardlow_2014}.

\begin{algorithm}[!t]
 \caption{Stochastic Spatio-Temporal Optimization}
 \begin{algorithmic}[1]
 \State \textbf{Function:} \textit{$\vTheta^* =$ \textbf{OptimizePolicyActuatorVars}($T$,$K$,$R$,$Z_0$,$N$,$\rho$,$\Delta t$,$\mu$,$\sigma_{\mu}^{(0)}$,$\vTheta^{(0)}$,$\vx_p^{(0)}$, $\gamma_\vTheta$,$\gamma_{\vx_c}$,$\gamma_{\vx_p}$)}
 \For {$k=1 \; \text{to} \; K$}
  \State Compute $\vm(\vx_p, \vx_c), M(\vx_p, \vx_c)$ $\forall$ $\vx_p \in \hat{\calD}$
 \For {$r=1 \; \text{to}\; R$}
 \For{$t=1 \;\text{to}\; T$}
     \State $\rd W_t^r \gets SampleNoise()$
     \State $Z_t^r \gets Propagate(Z_{t-1}^r,\vTheta^{(k)}, \rd W_t^r)$ via \cref{eq:EB_Hilbert_form}
     \State $u_t^r \gets SparseForwardPass(\vTheta^{(k)},Z_t^r)$
 \EndFor
 \EndFor
 \State $J^{0:R} \gets StateCost(Z_{0:T}^{0:R})$
 \State $N^{0:R}\gets \calN\big(u_{0:T}^{0:R}, \rd W_{0:T}^{0:R},  \vm(\vx)\big)$ via \cref{eq:noise_inner_product}
 \State $P^{0:R} \gets \calP\big(u_{0:T}^{0:R}, M(\vx)\big)$ via \cref{eq:policy_inner_product}
 \State $E^{0:R} \gets \calE(J^{0:R}, N^{0:R}, P^{0:R})$ as in \cref{eq:gibbs_final}
 \State $L \gets ComputeLoss(P^{0:R},N^{0:R},E^{0:R})$ via \cref{eq:importance_sampled_cost}
 \State Compute $\nabla_\vTheta L$ via backprop
 \State Compute $\nabla_{\vx_p} L$ via backprop
 \State Compute $\nabla_{\vx_c} L$ via backprop
 \State $\vTheta^{(k+1)} \gets GradientStep(\nabla_\vTheta L, \gamma_\vTheta, \vTheta^{(k)})$ via \cref{eq:mod_theta_update}
 \State $\vx_c^{(k+1)} \gets GradientStep(\nabla_{\vx_c} L, \gamma_{\vx_c}, \vx_c^{(k)})$ via \cref{eq:mod_xc_update}
 \State $\vv^{(k+1)} \gets GradientStep(\nabla_{\vx_p}L, \gamma_{\vx_p}, \vv^{(k)})$ via \cref{eq:mod_v_update}
 \State $\vx_p^{(k+1)} \gets SnapToGrid(\vv^{k+1})$ via \cref{eq:mod_xp_update}
 \EndFor
 \end{algorithmic}
 \label{Algorithm1}
\end{algorithm}

The method $GradientStep$ performs a gradient update of any gradient-based optimization algorithm. In this work we apply ADAM~\cite{kingma2014adam} gradient update variants of \cref{eq:mod_theta_update,eq:mod_xc_update,eq:mod_v_update} in all of our experiments. This version of $GradientStep$ is different than that of \cite{evans2020spatio} due to the modifications described in \cref{subsec:modified_opt}, namely there is no need to add a separate~\footnote{separate from the momentum native to the ADAM update} momentum term to help the gradients reach optimal values since we are now carrying the continuous-valued virtual variable $\vv$, which can change over iterations even when the true variable $\vx_p$ remains at the previous grid element due to the method $SnapToGrid$.

The use of different learning rates for each type of variable is often essential. The authors conjecture that the optimization landscape is typically more shallow for the actuator design than for the policy parameters. For most of the experiments, the actuator placement learning rate $\gamma_{x_p}$ is set to about 30 times larger than the policy network learning rate $\gamma_\vTheta$, however this can be dependent on the problem, selection of number of actuators, and policy parameter initialization type (e.g. Xavier vs zeroes).

 

\section{Policy \& Co-Design Optimization of Simulated Robotics PDEs} \label{sec:results_rss}

In \cite{evans2020spatio} the approach was applied to four simulated SPDE experiments to simultaneously place actuators and optimize a policy network. Each experiment used less than 32 GB RAM, and was run on a desktop computer with a Intel Xeon 12-core CPU with a NVIDIA GeForce GTX 980 GPU. The code was written to operate inside a Tensorflow graph~\cite{tensorflow} to leverage rapid static graph computation, as well as sparse linear algebra operations used by $SparseForwardPass$~\cite{Evans2019IDVRL}. The first two experiments involved a reaching task, where the \acp{SPDE} are initialized at a zero initial condition over the spatial region and must reach certain values at pre-specified regions of the spatial domain. The last two experiments involved a suppression task, where some non-zero initial condition must be suppressed on desired regions.

The data that was used for training was generated by a spatial central difference, semi-implicit time discretized version of each \ac{SPDE}. These schemes are described in detail in~\cite[Chapter 3 \& 10]{lord_powell_shardlow_2014}. Each experiment had all actuator locations initialized by sampling from a uniform distribution on $[0.4a, 0.6a]$, where $a$ denotes the spatial size. For 3500 iterations of the algorithm, run times for the most complicated system--the Euler-Bernoulli equation--were about 15 hours. Details of the experiments and videos of the controlled systems can be found in the provided links~\footnote{Supplement: \href{https://tinyurl.com/yc7fq3lc}{tinyurl.com/yc7fq3lc} $\vert$ Video: \href{https://youtu.be/yo48a6JqKE0}{https://youtu.be/yo48a6JqKE0}}. We encourage the interested reader to contact the first author for code.

Each of the experiments in this section utilized \acp{FNN} for the nonlinear policy $\varphi(h;\Theta)$, built in a Tensorflow graph~\cite{tensorflow} with two hidden layers of ReLU neurons. All policy network weights were initialized with the Xavier initialization~\cite{glorot2010understanding} and trained with ADAM~\cite{kingma2014adam}. In every experiment the function $\vm(\vx)$ was modeled as a Gaussian-like exponential function with the means co-located with the actuator locations, and considered a state cost functional of the form
\begin{equation}\label{supeq:cost}
    J := \sum_{t} \sum_{x} \;\kappa \big(h_{\text{actual}}(t,x) - h_{\text{desired}} (t,x)\big)^2 \cdot \mathbbm{1}_{S}(x)
\end{equation}
where $\kappa$ is a state cost weightm, and $\mathbbm{1}_{S}(x)$ is defined by
\begin{equation} \label{supeq:1Dheat_indicator}
\mathbbm{1}_S(x) :=
\begin{cases}
1,  \quad \text{if } x \in S  \\
0, \quad  \text{otherwise},
\end{cases}
\end{equation}
where $S$ is the spatial subregion on which the desired profile is defined. These desired spatial regions were different for each experiment, and are depicted as green bars in the associated figures.

\begin{figure*}[!t]
    \subfigure[Controlled Contour]{\hspace{-0.0cm}\includegraphics[width=0.58\linewidth]{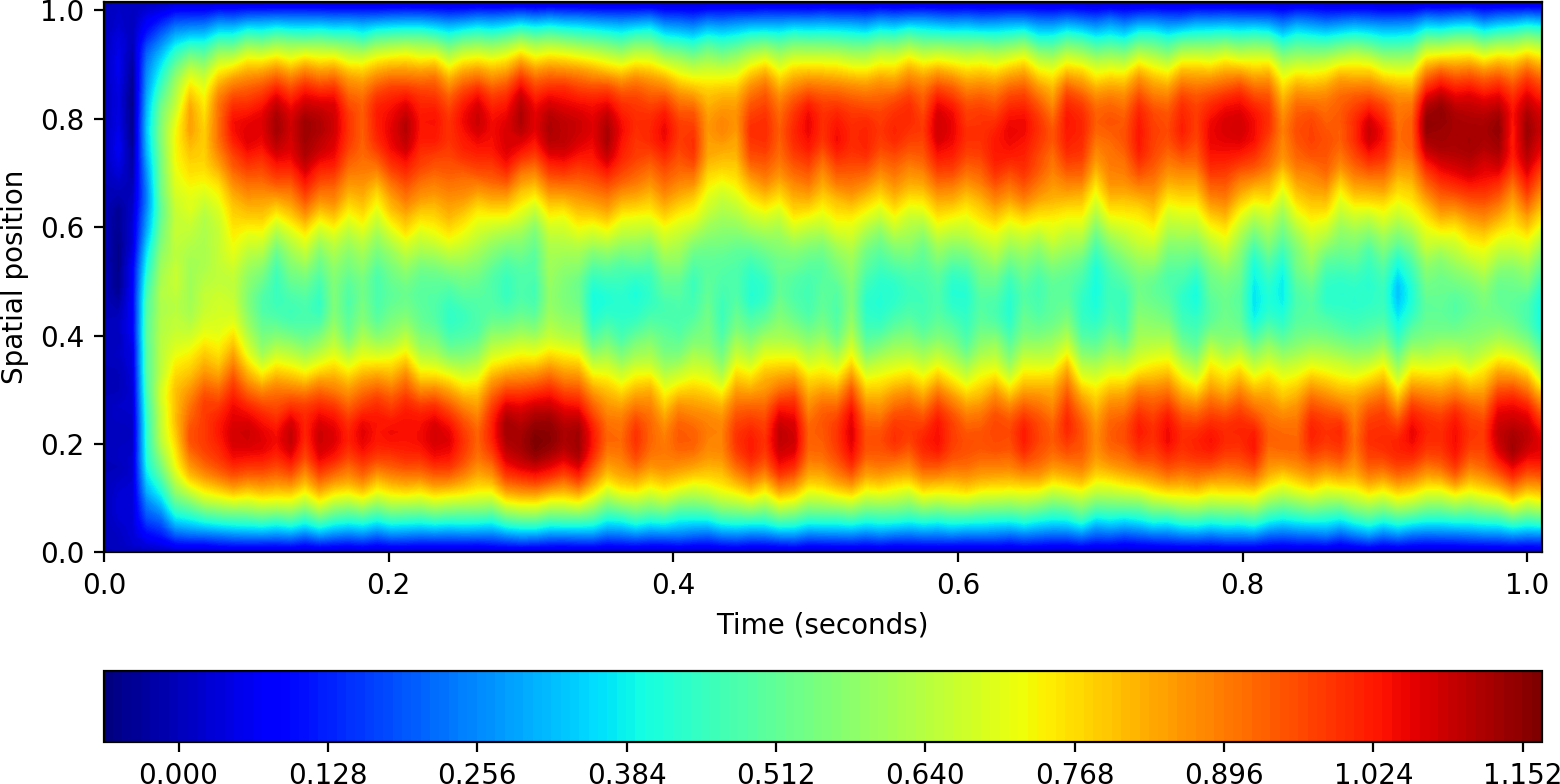}}
    \subfigure[Final Time Snapshot]{\hspace{0.02cm}\includegraphics[width=0.406\linewidth]{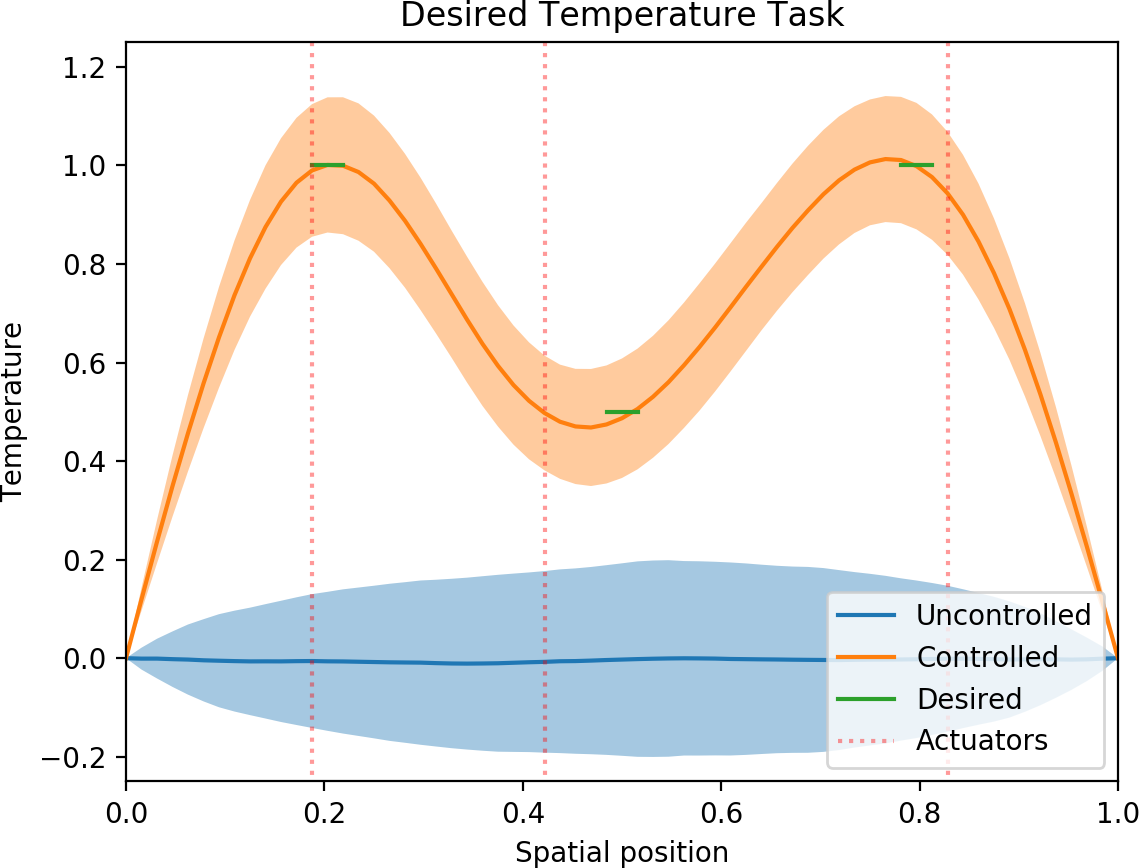}}
    \caption{Heat Equation Temperature Reaching Task. (a) controlled contour plot of a randomly selected trajectory rollout where color represents temperature, (b) final time snapshot comparing to the uncontrolled system. Mean trajectories are represented with a solid line, while a 2$\sigma$ standard deviation is represented with a shaded region.} 
    \label{fig:heat}
\end{figure*}

\begin{figure*}[!t]
    \subfigure[Controlled Contour]{\hspace{-0.0cm}\includegraphics[width=0.58\linewidth]{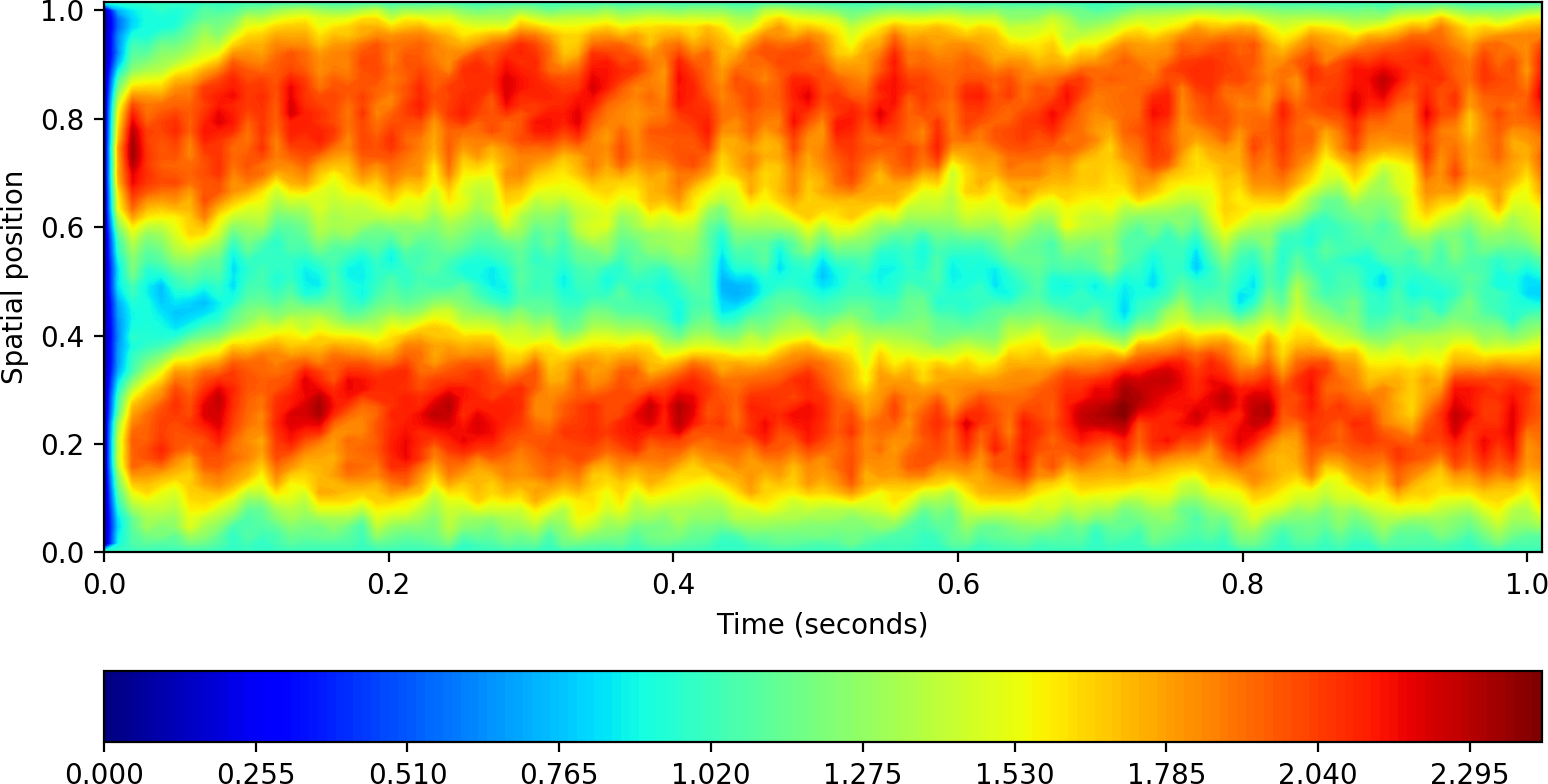}}
    \subfigure[Final Time Snapshot]{\hspace{0.02cm}\includegraphics[width=0.406\linewidth]{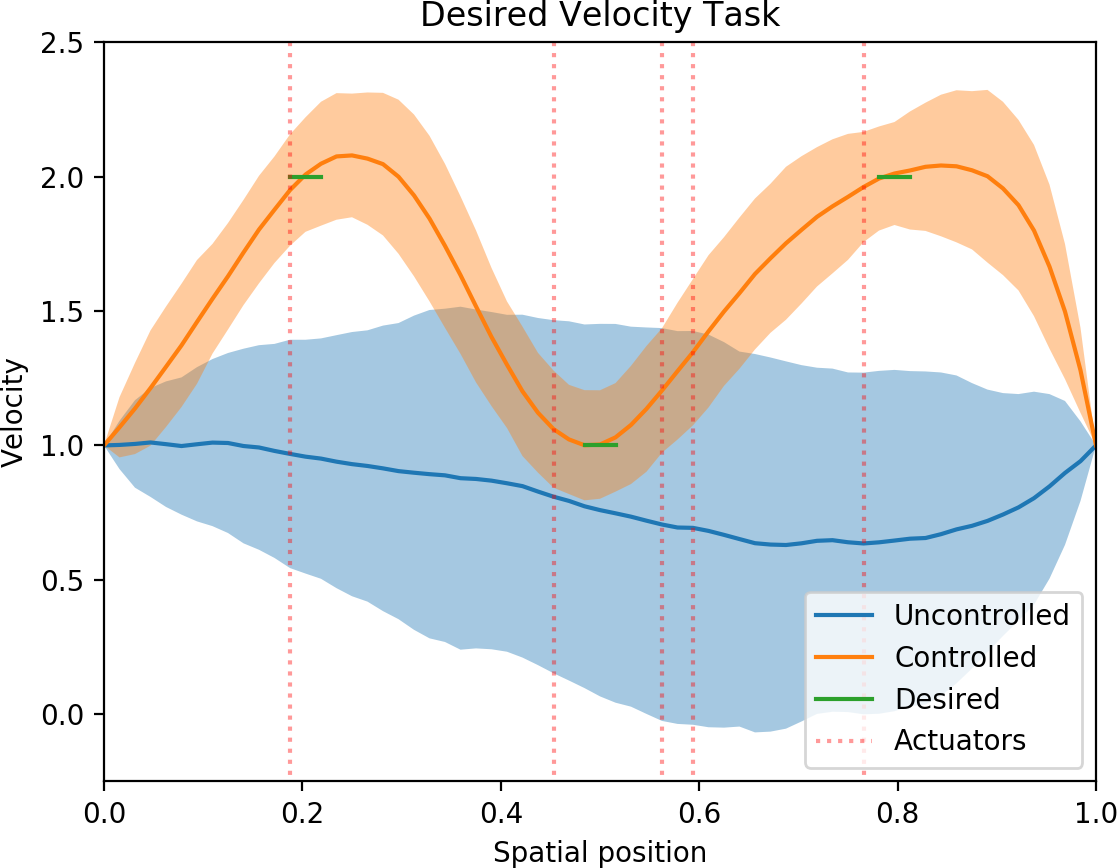}}
    \caption{Burgers Velocity Reaching Task. (a) controlled contour plot of a randomly selected trajectory rollout where color represents velocity, (b) final time snapshot comparing to the uncontrolled system. Mean trajectories are represented with a solid line, while a 2$\sigma$ standard deviation is represented with a shaded region.} 
    \label{fig:burgers}
\end{figure*}

The first experiment was a temperature reaching task on the \ac{1D} Heat equation with homogeneous Dirichlet boundary conditions, given in \textit{fields representation} by
\begin{equation}
    \begin{split}
    \partial_t h(t,x) &= \epsilon \partial_{xx} h(t,x) + \vm(\vx)^{\top} \, \varphi\big(h(t,x);\Theta\big)  + \frac{1}{\sqrt{\rho}} \partial_t W(t,x) ,\\
    h(t,0) &= h(t,a) = 0, \\
    h(0,x) &= h_0(x),
    \end{split}
\end{equation}
where $\epsilon$ is the thermal diffusivity parameter. The results of $3000$ iterations of optimization with $200$ trajectory rollouts per iteration are depicted in \cref{fig:heat}. The task was to raise the temperature at regions specified in green to specified values depicted in the figure.

The next experiment was a velocity reaching task on the Burgers equation with non-homogenous Dirichlet boundary conditions, given in \textit{fields representation} by
\begin{equation} \label{supeq:BurgersSPDE}
\begin{split}
\partial_t h(t, x) &= - h(t,x) \partial_x h(t, x) + \epsilon \partial_{xx} h(t,x) + \vm(\vx)^{\top} \, \varphi(h;\Theta) + \frac{1}{\sqrt{\rho}} \partial_t W(t), \\
h(t,0) &= h(t,a) = 1.0,\\ 
h(0,x) &= h_0(x), 
\end{split}
\end{equation}
where the parameter $\epsilon$ is the viscosity of the medium. The results of $3500$ iterations with $100$ trajectory rollouts per iteration are depicted in \cref{fig:burgers}. The Burgers equation is often used as a simplified model of fluid flow, however Burgers-like reaction-advection-diffusion \ac{PDE} are also often used to describe swarms of robotic systems \cite{elamvazhuthi2018pde}. The Burgers equation has a nonlinear advection term, which produces an apparent rightward motion. The algorithm appears to have taken advantage of the advection for actuator placement in order to solve the task with lower control effort.

The heat equation is a pure diffusion \ac{SPDE}, while the Burgers equation shares the diffusion term with the Heat equation with an added advection term. The results of the Heat and Burgers experiments show actuator locations that take advantage of the natural behavior of each \ac{SPDE}. In the case of the Heat equation, actuators are near the desired regions such that the temperature profile can reach a flat peak of the diffusion at the desired profile. In the case of the Burgers equation, the advection pushes towards the right end of the space, thus forming a wave front that develops at the right end, but leaves the left end dominated by the diffusion term. This is again reflected in the placement of actuators. The first actuator is near the desired region just as the actuators in the Heat \ac{SPDE}, while two of the actuators between the center and the right region are located to be able to control the amplitude and shape of the developing wave front so as to produce a flat peak that aligns with the desired region at the desired velocity. The central desired region is flanked on both sides by actuators that are nearly equidistant, in order to produce another desired flat velocity region at this location.

The third experiment was a voltage suppression task on the Nagumo equation with homogeneous Neumann boundary conditions,
given in \textit{fields representation} by
\begin{align}
\partial_t h(t,x) &= \epsilon \partial_{xx} h(t,x) + h(t,x)\big(1-h(t,x)\big)\big(h(t,x)-\alpha\big) +  \vm(\vx)^{\top} \, \varphi(h;\Theta) \nonumber \\
&\quad + \frac{1}{\sqrt{\rho}} \partial_t W(t,x)  \nonumber \\ 
h_x(t,0) &= h_x(t,a) = 0 \label{eq:nagumo}\\ 
h(0,x) &= \bigg(1+\exp\Big(-\frac{2-x}{\sqrt[]{2}}\Big)\bigg)^{-1} \nonumber ,
\end{align}
where the parameter $\alpha=-0.5$ determines the speed of a wave traveling down the length of the extent and $\epsilon=1.0$ determines the rate of diffusion. The Nagumo equation is often used in neuroscience as a model of the propagation of voltage across an axon in neuronal activation~\cite{lord_powell_shardlow_2014}. However, it has also been used in robotics applications, such as in \cite{aidman2008coupled}, where it was used to describe robot navigation in crowded environments.

\begin{figure*}[t]
    \subfigure[Controlled Contour]{\hspace{-0.0cm}\includegraphics[width=0.58\linewidth]{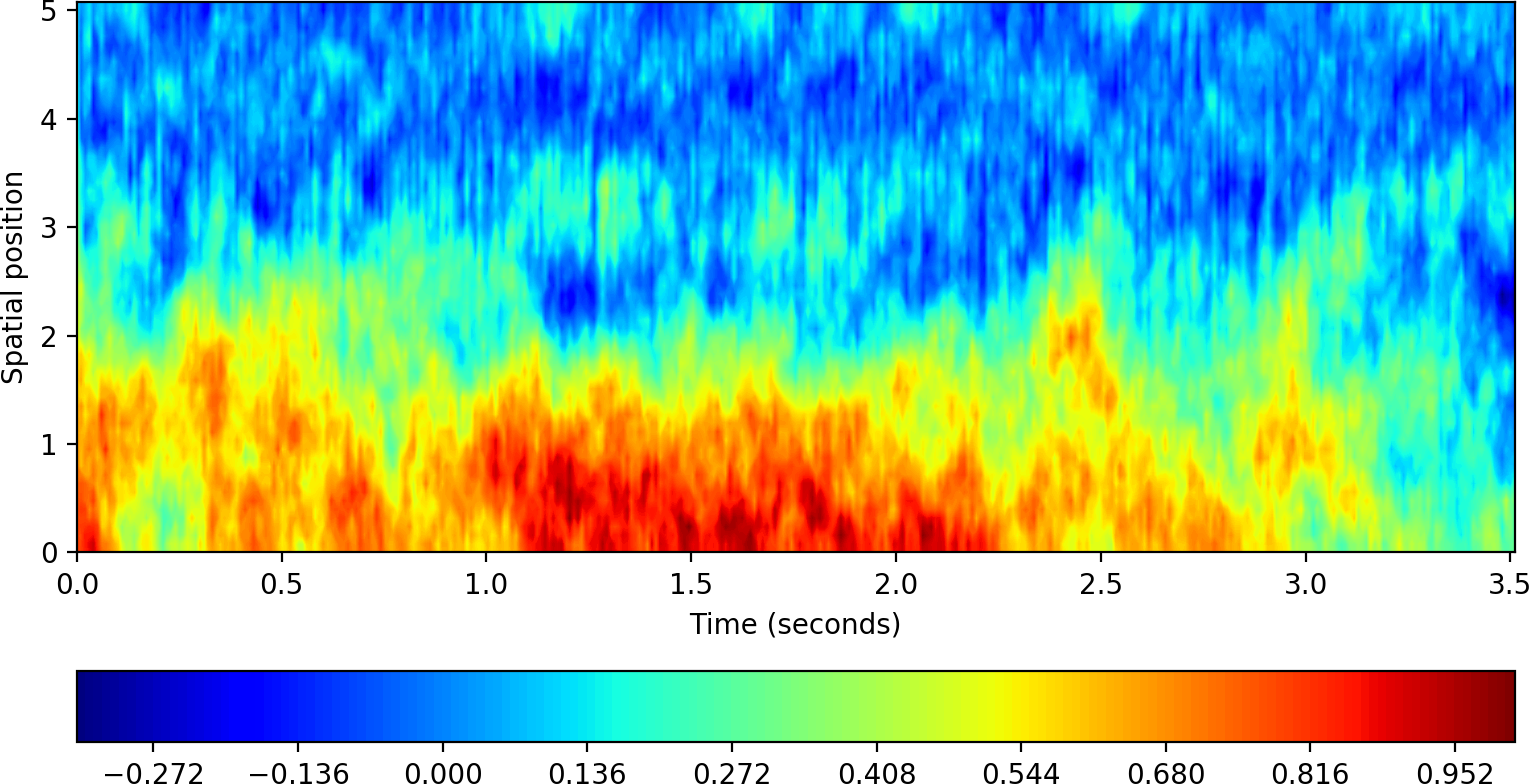}}
    \subfigure[Final Time Snapshot]{\hspace{0.02cm}\includegraphics[width=0.406\linewidth]{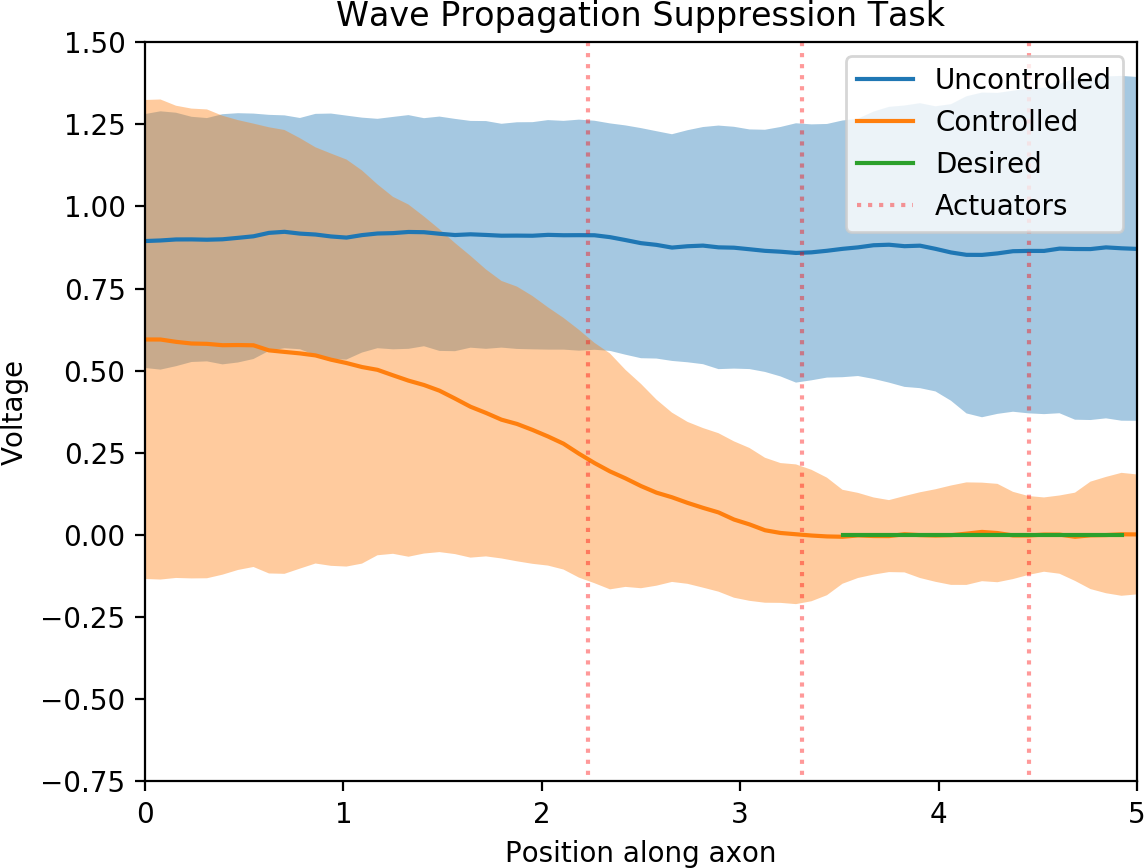}}
    \caption{Nagumo Suppression Task. (a) controlled contour plot of a randomly selected trajectory rollout where color represents voltage, (b) final time snapshot comparing to the uncontrolled system. Mean trajectories are represented with a solid line, while a 2$\sigma$ standard deviation is represented with a shaded region.} 
    \label{fig:nagumo}
\end{figure*}

\begin{figure*}[t]
    \subfigure[Controlled State Costs]{\hspace{-0.0cm}\includegraphics[width=0.402\linewidth]{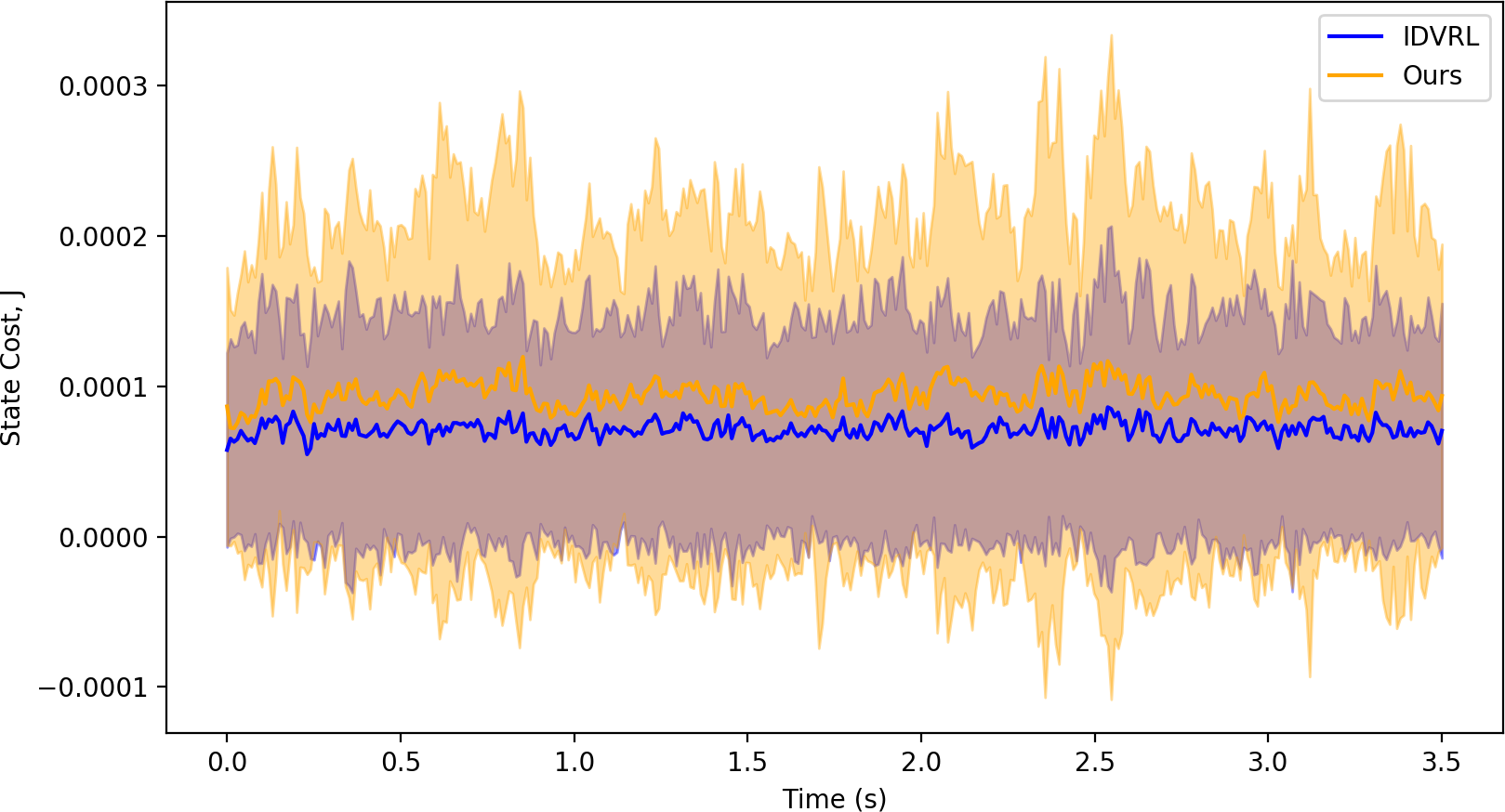}}
    \subfigure[Control Signals]{\hspace{0.02cm}\includegraphics[width=0.287\linewidth]{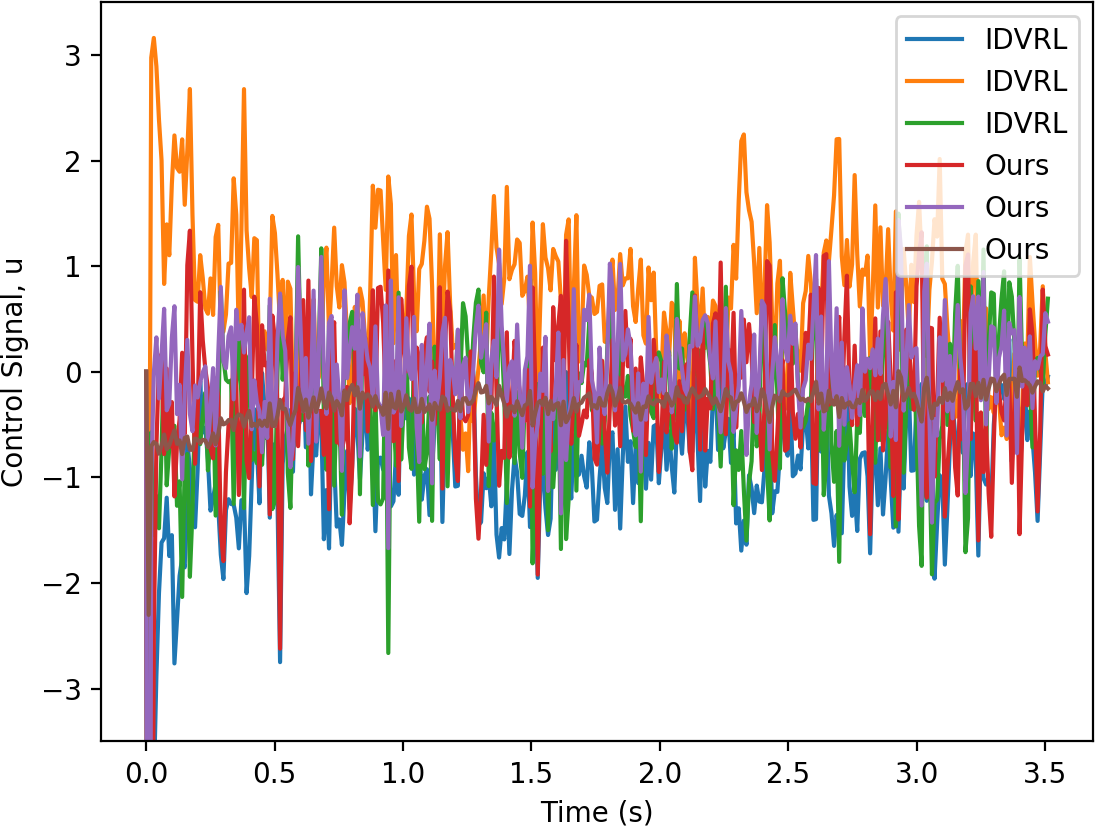}}
    \subfigure[Final Time Snapshot]{\hspace{0.02cm}\includegraphics[width=0.295\linewidth]{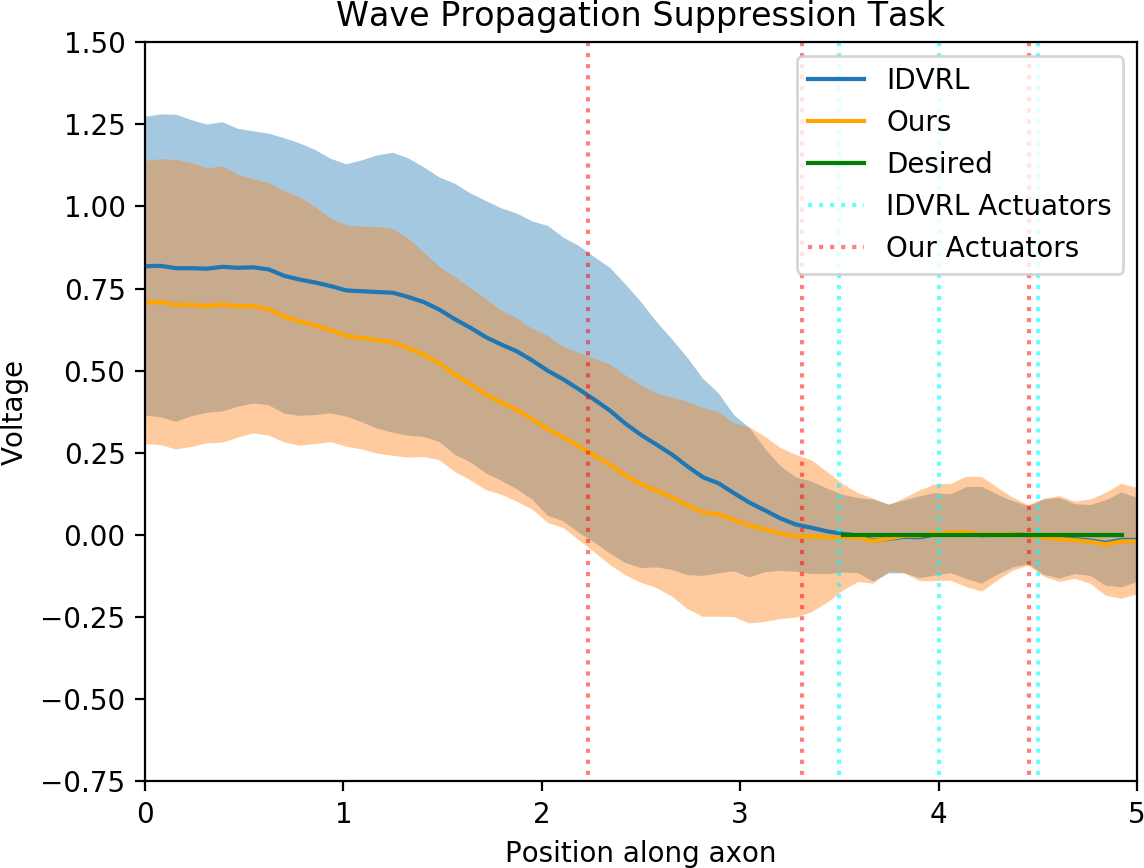}}
    \caption{Nagumo Suppression Task Comparison Plots. (a) controlled state cost plot, where solid lines denote mean, and shaded regions denote a $2 \sigma$ standard deviation, (b) Control signal comparison plot, where lines represent mean behavior, and (c) Final time snapshot comparing the actuators placed by our approach and actuators placed by a human expert with policy optimization by IDVRL.} 
    \label{fig:nagumo_comparison}
\end{figure*}

The joint policy and co-design optimization results are depicted in \cref{fig:nagumo}. The task was to suppress an initial voltage on the left end, that without intervention propagates toward the right end, as shown by the uncontrolled trajectories. The Nagumo equation in \cref{eq:nagumo} is composed of a diffusive term and a 3rd-order nonlinearity, making this equation the most challenging \ac{1D} \ac{SPDE} from a nonlinear control perspective. Despite this, our approach was able to simultaneously place actuators and provide control such that the task was solved. The algorithm was run for 2000 iterations, and demonstrates actuator placement optimization that takes advantage of the natural system behavior. This task was also the most challenging due to the significantly longer planning horizon of 3.5 seconds, as compared to the 1.0 second planning horizon of all the other experiments.

In order to validate our proposed approach, we compared the actuator locations that the algorithm found after optimization to the actuator locations that were hand placed by a human expert for the simulated experiments conducted in our prior work~\cite{Evans2019IDVRL}. To have a valid comparison, we ran the IDVRL algorithm for both sets of actuator locations. Figure \ref{fig:nagumo_comparison} reports these results. The left figure shows that the state costs for each are almost identical. Note that the scale here is $10^{-4}$. The center figure shows the control signals for each actuator, for each method, and demonstrates that for almost identical state cost values, the control effort for each actuator with our approach is lower on average. The calculated average control signal magnitudes for human-placed actuators are 3.3 times higher than those placed by our method. The third plot shows the voltage profile at the final time. We hypothesize that the lower control effort is due to the control over the shape of the spatially propagating signal, enabling it to have a smoother transition into the desired region. While the penalty of this actuator placement is a slightly higher variance on the desired region, the choice appears correct given the result.

The final task conducted in \cite{evans2020spatio} was an oscillation suppression task on the Euler-Bernoulli equation with Kelvin-Voigt damping given in \cref{eq:EB_SPDE}, and is depicted in \cref{fig:EB}. As shown, the initial condition prescribes spatial oscillations, that then oscillate temporally. The second-order nature of the system creates offset and opposite oscillations in the velocity profile, that in turn produce offset and opposite oscillations in the position profile. Without interference, the oscillations proceed over the entire time window, as shown in \cref{fig:EB} (a). As shown in \cref{fig:EB} (b), our approach successfully suppresses these oscillations, which die out quickly under the given control policy. In this experiment, the actuators remained inside the initialized actuator placement region $[0.4a, 0.6a]$ prescribed for all experiments.

\begin{figure*}[t]
    \subfigure[\label{fig:EB_a}Uncontrolled Contour]{\hspace{-0.0cm}\includegraphics[width=0.494\linewidth]{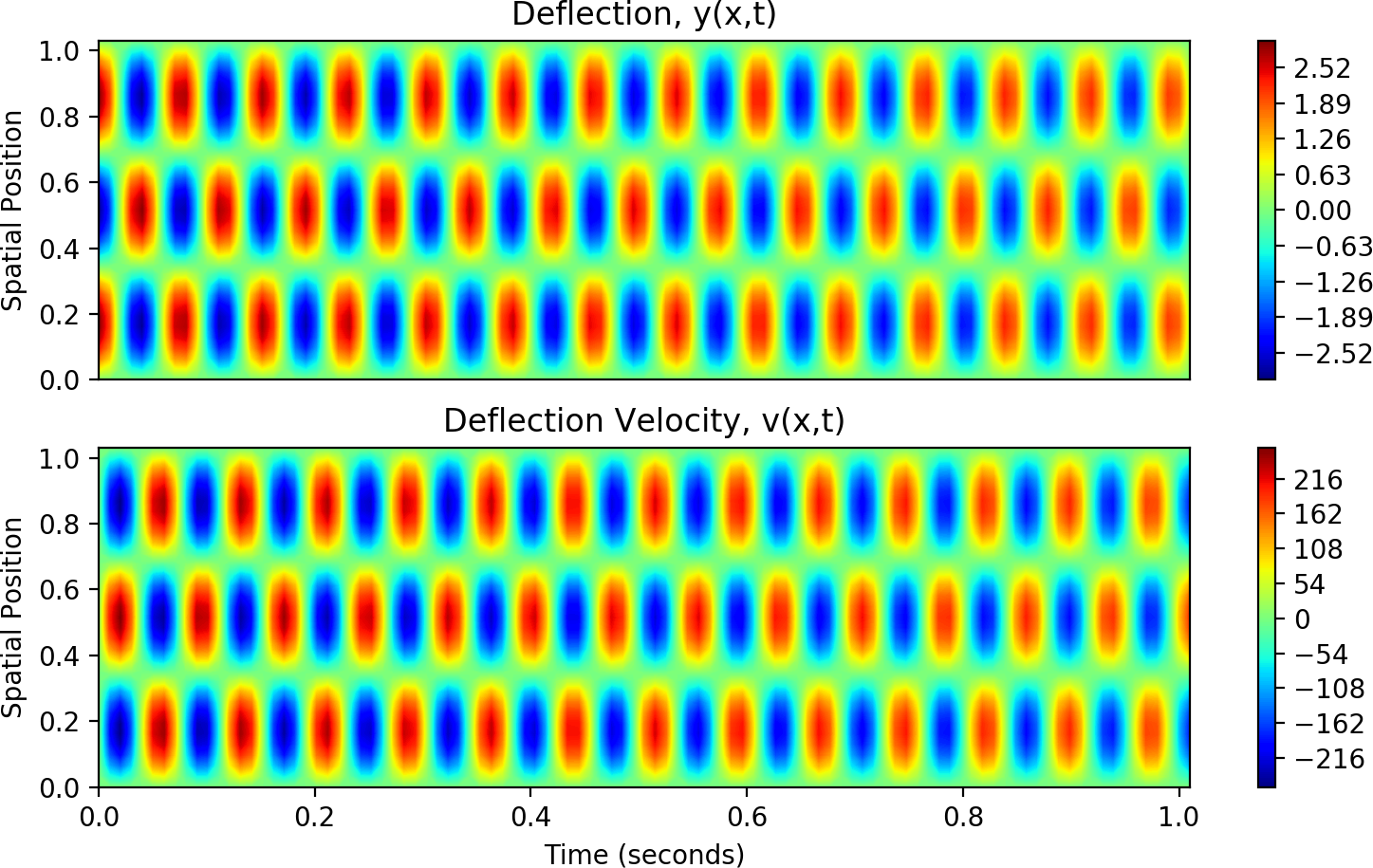}}
    \subfigure[\label{fig:Eb_b}Controlled Contour]{\hspace{0.05cm}\includegraphics[width=0.494\linewidth]{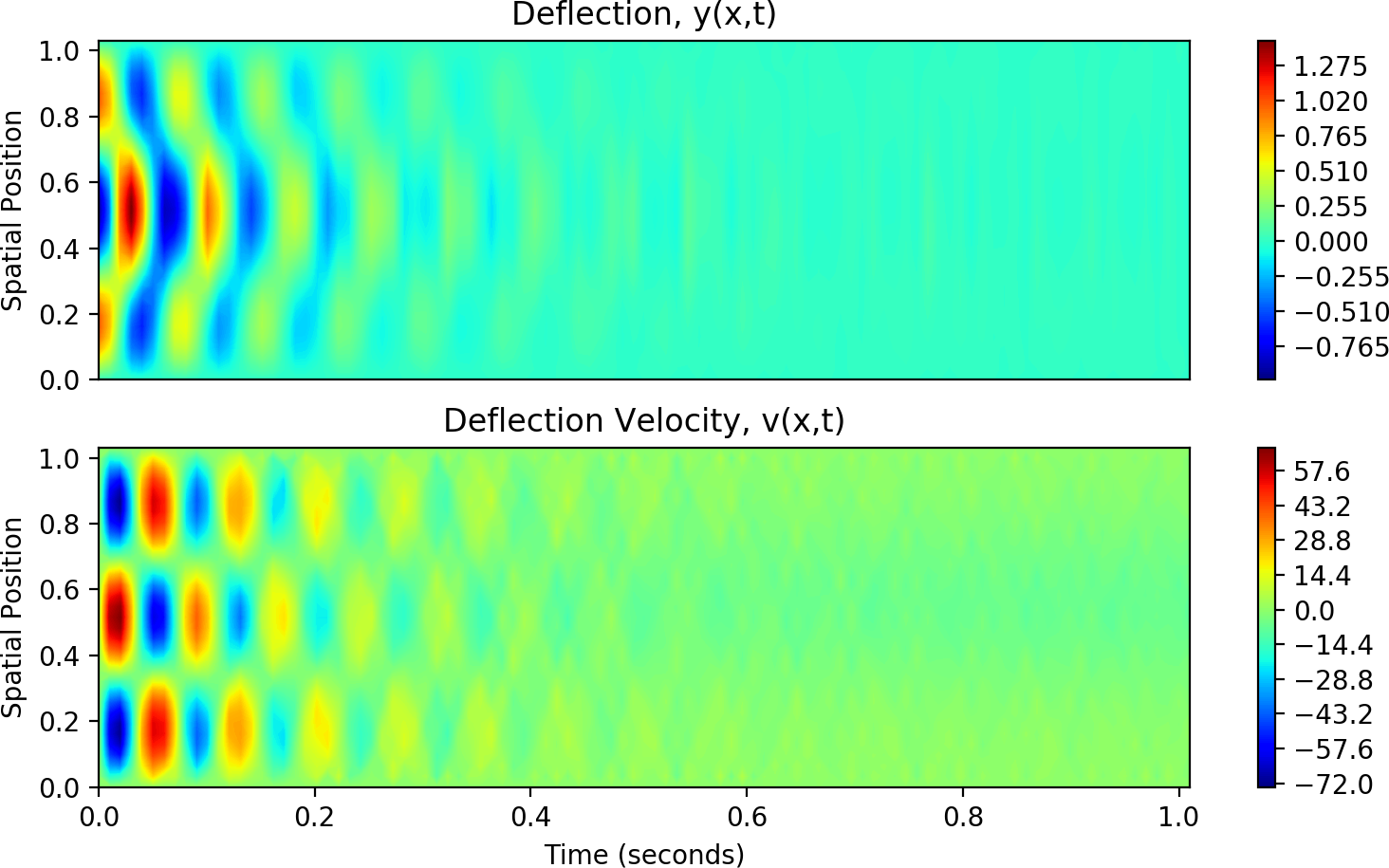}}
    \caption{Euler-Bernoulli Suppression Task. (a) Uncontrolled contour plot (b) Controlled contour plot. In both plots, color represents deflection on top, and deflection velocity on bottom.} 
    \label{fig:EB}
\end{figure*}

The Euler-Bernoulli oscillation suppression task is in fact very challenging and complex. Producing a control signal at an actuator location that is in phase with the velocity oscillations will amplify the oscillations, leading to a divergence. The actuator location and control signal from the policy network must work in concert to produce a control signal out of phase with the velocity and matching its frequency, which is time varying due to control. This time-varying frequency is depicted in \cref{fig:EB} (b).

Each of the above \ac{1D} experiments have unique challenges and in most cases the spatio-temporal problem space produces a joint policy optimization and actuator co-design problem that is littered with local minima. These experiments demonstrate that the proposed approach can jointly optimize a policy network and actuator design. These results and the overall performance of the algorithm are indicative that this approach may enable actuator design on problem spaces where a human has little to no prior knowledge to rely on in a-priori designing actuation to solve the problem by hand.

\subsection{Scaling to Higher Dimensions} \label{sec:results_scaling}

The above experiments motivated the novel experiments presented here, which scale policy and actuator co-design optimization to large \ac{2D} problem spaces. The primary challenges with scaling to higher dimensions are related to the memory storage requirements of large computational graphs. As discussed in the previous section, each of the prior experiments were performed on a static TensorFlow graph, which in many cases has an advantage in runtime performance, yet requires significant memory pre-allocation. Scaling policy and actuator co-design optimization to 2D spaces often requires more memory storage than $64$ GB in a static graph setting. As such, it is recommended to balance dynamic allocation with static graph computation.

For this task, the goal is to control the \ac{2D} Heat equation with homogeneous Dirichlet boundary conditions. The goal is to raise the temperature in four outer regions to $T=1.0$ and raise the temperature in one central region to $T=0.5$. The results of $5000$ iterations, with $100$ rollouts per iteration are depicted in \cref{fig:heat2d}. Similar to the \ac{1D} heat equation, here the pure diffusive nature of the Heat equation is best leveraged by placing actuators as close to the desired regions as possible, as is demonstrated by the algorithm in this case. A video of the controlled state evolution is available at the link provided \footnote{Video: \href{https://youtu.be/yo48a6JqKE0}{https://youtu.be/yo48a6JqKE0}}.

The \ac{SPDE} is spatially discretized using a central difference discretization into $25$ grid points on each axis for a total of $625$ states, and is temporally discretized with a semi-implicit time discretization.  This problem has a dramatic increase in scale compared to the \ac{1D} examples given above, which typically had $64$ states.

\begin{figure*}[t]
    \subfigure[Initial Time Contour]{\includegraphics[height=0.288\linewidth]{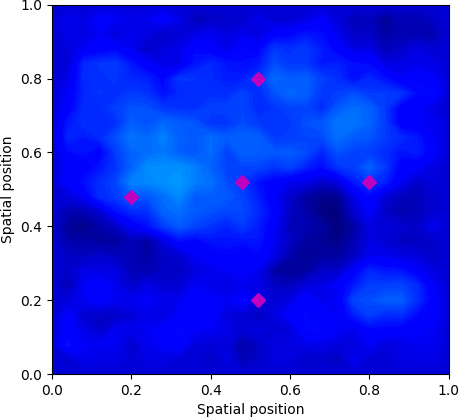}}
    \subfigure[Mid-Time Contour]{\hspace{0.05cm}\includegraphics[height=0.285\linewidth]{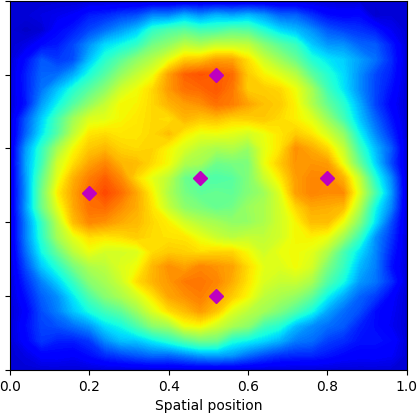}}
    \subfigure[Final Time Contour]{\hspace{0.05cm}\includegraphics[height=0.285\linewidth]{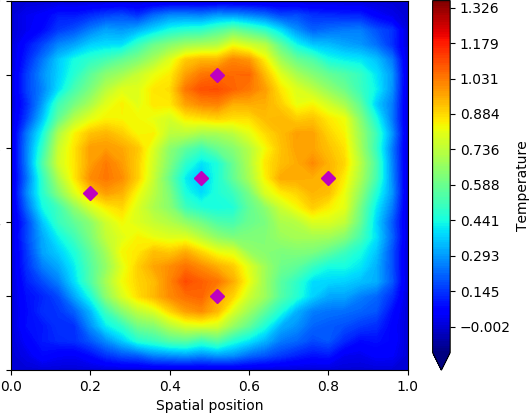}}
    \caption{Controlled 2D Heat Equation Contours of a random trajectory rollout with actuators denoted in magenta and color spectrum denoting temperature. (a) initial time contour with spatially random initial condition (b) half-way time contour (c) final time contour. Videos of a random trajectory rollout of the controlled system evolving in time can be found at \href{https://youtu.be/yo48a6JqKE0}{https://youtu.be/yo48a6JqKE0}}
    \label{fig:heat2d}
\end{figure*}

\begin{figure*}[t]
    \subfigure[State Cost Convergence]{\includegraphics[width=0.492\linewidth]{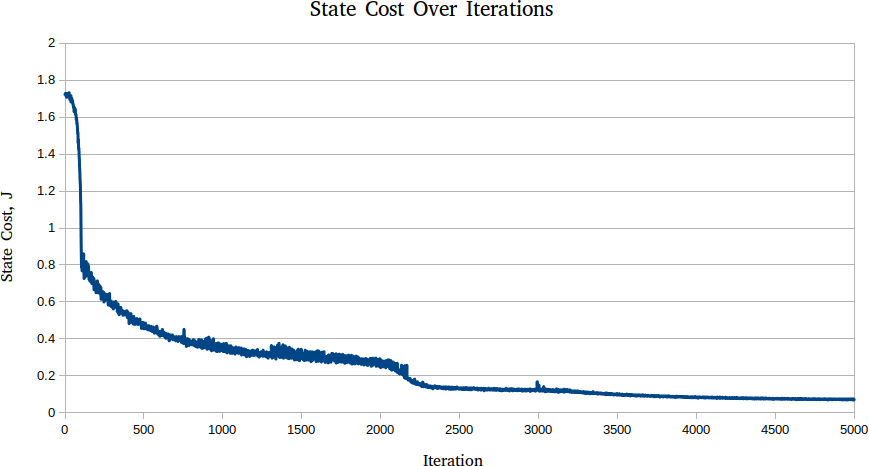}}
    \subfigure[Loss Convergence]{\hspace{0.1cm}\includegraphics[width=0.492\linewidth]{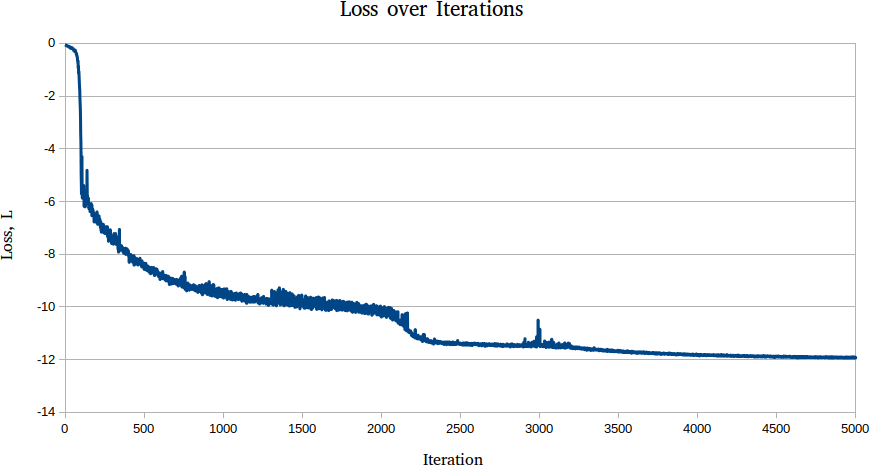}}
    \caption{Convergence of State Cost and Loss for 2D Heat Equation. (a) state cost over iterations for $5000$ iterations (b) loss over iterations for $5000$ iterations.}
    \label{fig:heat2d_convergence}
\end{figure*}

In this case five actuators were provided to the system to achieve this task, and were all intitialized by sampling $x$ and $y$ locations from a uniform distribution over $[0, a]$, where $a$ is the side length. The nonlinear policy network for this experiment utilized two convolutional layers, two max-pool layers, and a fully connected output layer, with ReLu activations throughout. The policy network weights were initialized with the Xavier initialization~\cite{glorot2010understanding}.

The convergence behavior of the algorithm over $5000$ iterations for the obtained solution is depicted in \cref{fig:heat2d_convergence}. As depicted, there is a close correlation between the behavior of the state cost and the behavior of the loss, which is desirable in many cases. However, in many cases this approach can violate strict proportionality between the loss and the state cost in the near term in order to obtain dramatic state cost improvements in the long term. This is reported in greater detail in our prior work~\cite{Evans2019IDVRL}.





\subsection{Policy \& Co-Design Optimization of a Soft Robotic Limb} \label{sec:results_spring_mass}

\begin{figure*}[t]
    \subfigure[Uncontrolled X Final Time Contour]{\includegraphics[width=0.505\linewidth]{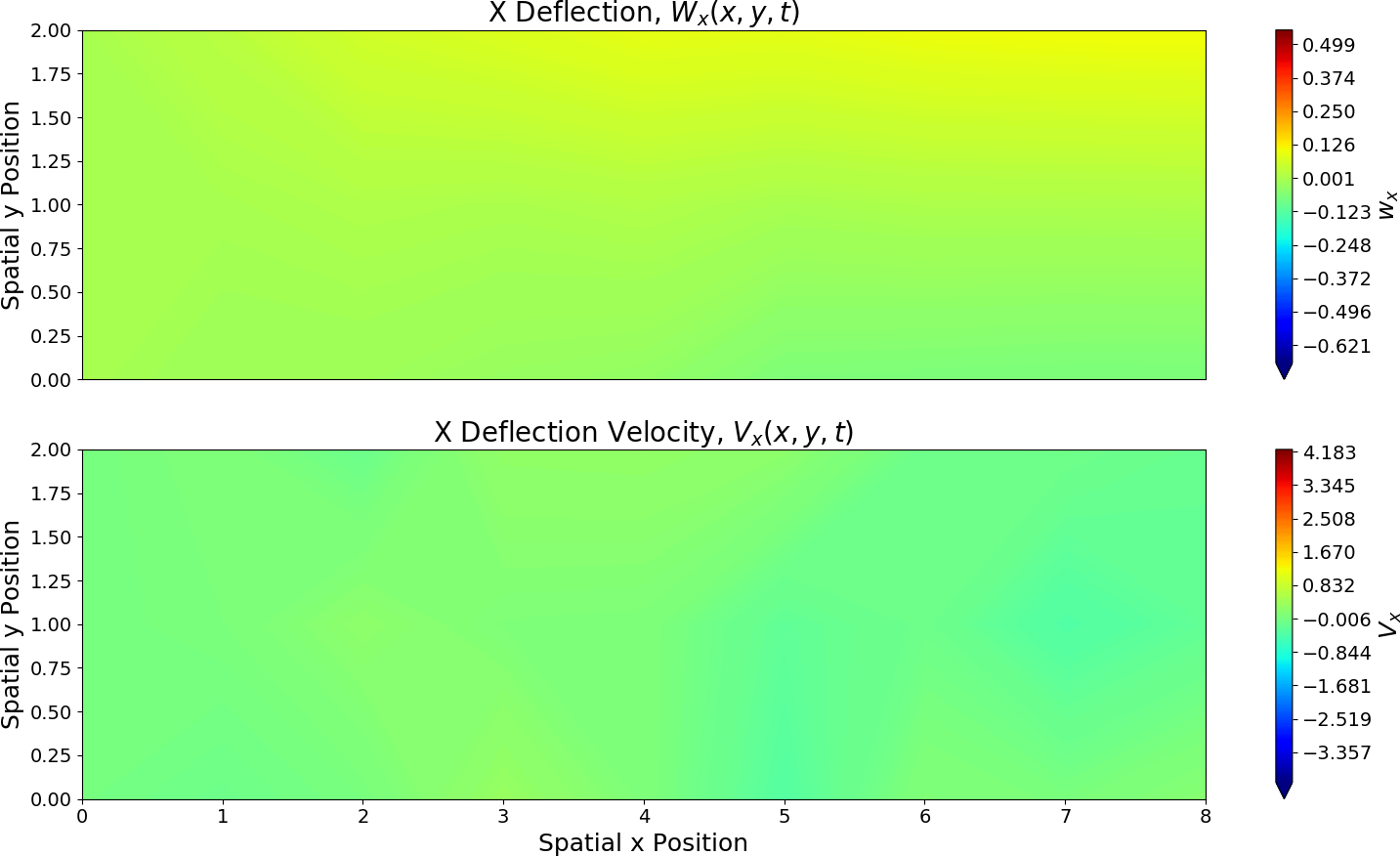}}
    \subfigure[Uncontrolled Y Final Time Contour]{\hspace{0.13cm}\includegraphics[width=0.48\linewidth]{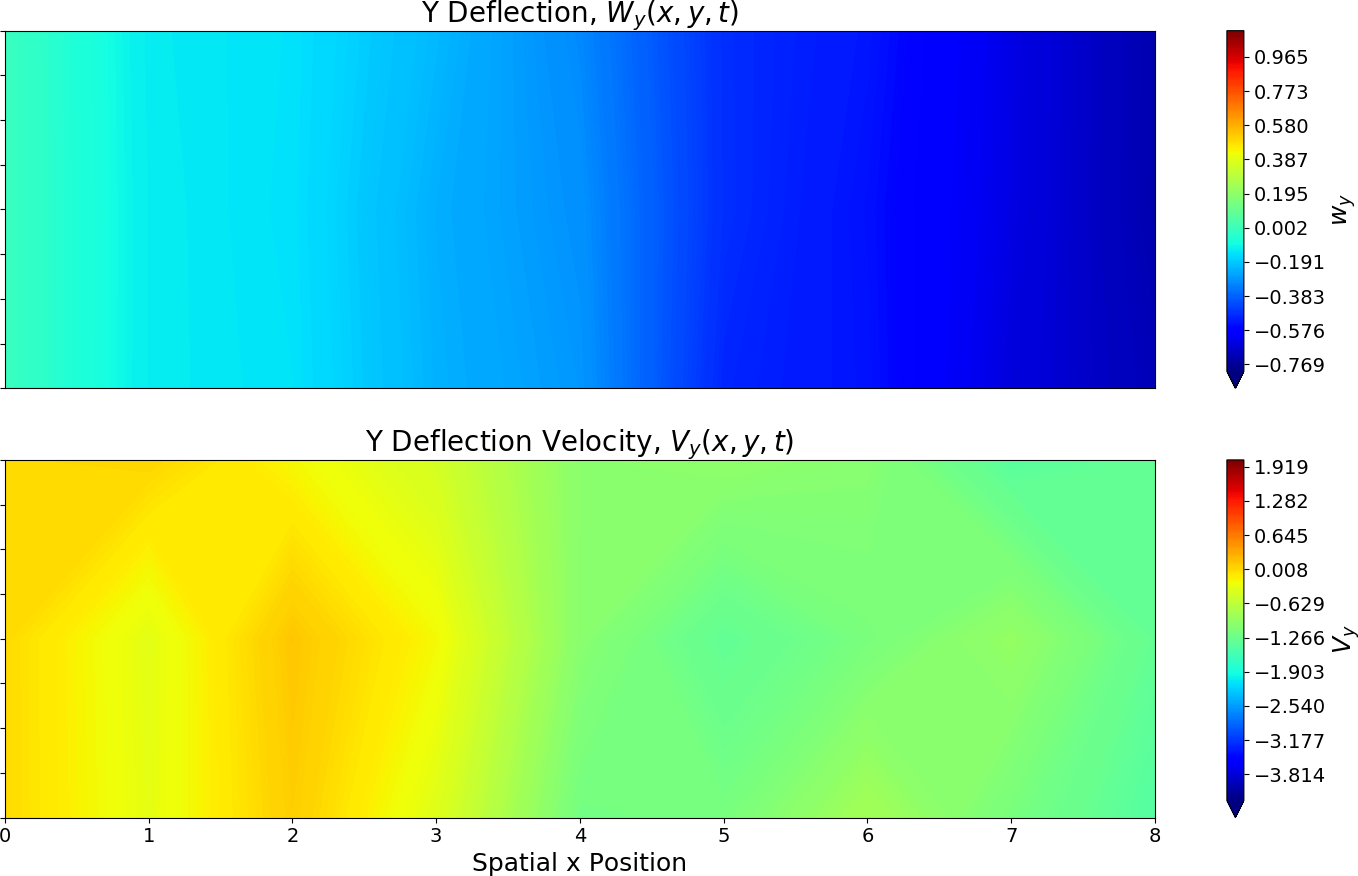}}
    
    \subfigure[Controlled X Final Time Contour]{ \newline\includegraphics[width=0.505\linewidth]{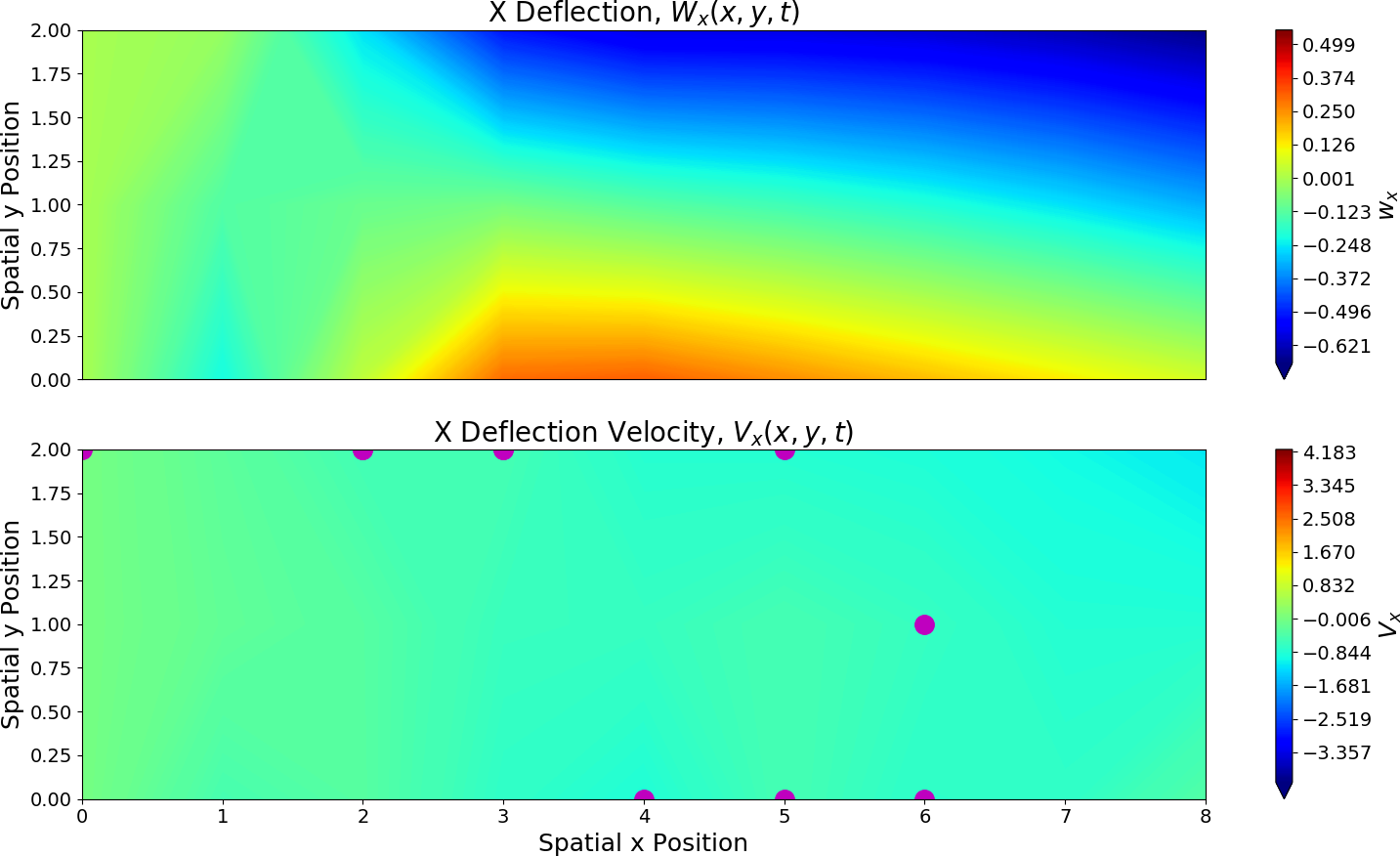}}
    \subfigure[Controlled Y Final Time Contour]{\hspace{0.13cm}\includegraphics[width=0.48\linewidth]{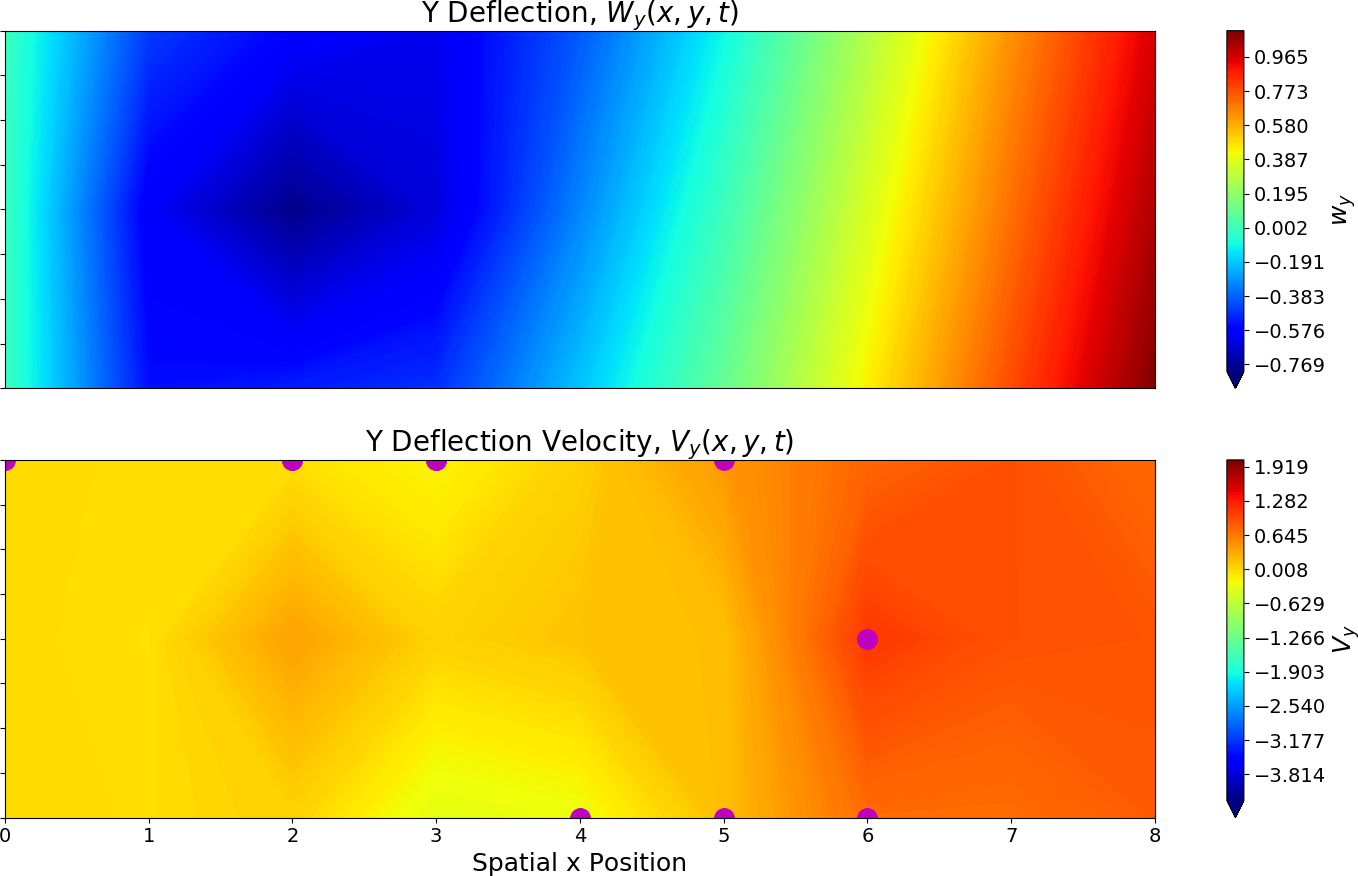}}
    \caption{2D Soft Arm Reaching Task contour plots of a random trajectory rollout at final time, where color represents magnitude. Purple circles represent actuators on the controlled system placed by actuator co-design optimization. (a) uncontrolled final time snapshots of x deflection and x deflection velocity (b) uncontrolled final time snapshots of y deflection and y deflection velocity (c) controlled final time snapshots of x deflection and x deflection velocity (d) controlled final time snapshots of y deflection and y deflection velocity. Videos of a random trajectory rollout of the controlled and uncontrolled system evolving in time can be found at \href{https://youtu.be/yo48a6JqKE0}{https://youtu.be/yo48a6JqKE0}.}
    \label{fig:spring_mass}
\end{figure*}

In the final experiment, the algorithm was applied to a 2D soft manipulator governed by the dynamics in \cref{eq:continuumPDE}. These dynamics exhibit complex nonlinear behavior that in many ways present the most challenging task that has been conducted with our approach; the \ac{2D} \ac{PDE} dynamics are nonlinear, 2nd-order, and stochastic.

For this experiment the task was to jointly optimize the policy and placement of actuators such that the soft limb deflected by one unit vertically while maintaining initial tip position in $x$, subject to a highly exaggerated gravitational force two orders of magnitude larger than normal. The exaggerated gravitational force models an external force preventing task completion and forces better actuation design performance for task completion. The result of $4000$ iterations of the algorithm  with $50$ rollouts per iteration are depicted in \cref{fig:spring_mass}.

The system is spatially discretized by lumping it into a Cartesian grid of point mass particles. Each particle, identified by horizontal index $i$ and vertical index $j$, has a position $\vx_{i,j}$ and a velocity $\vv_{i,j}$. The term $\rho_m \frac{\partial^2 \vs}{\partial t^2 }$ in \cref{eq:continuumPDE} is replaced by $\rho_m \frac{\partial^2 \vx_{i,j}}{\partial t^2 }$. The resting length between adjacent particles is taken to be a constant length $l$. The strain divergence term $\text{div}(\sigma)$ in \cref{eq:continuumPDE} can be approximated by finite difference discretizations of the divergence operator and strain terms. A detailed derivation and discussion of this discretization scheme can be found in \cite{etzmuss2003deriving}. In this case the system has position and velocity states for each dimension over the \ac{2D} region, totaling $108$ states. Temporal discretization is accomplished by solving the spatially discretized system with the explicit Euler method.

Just as in the case of the \ac{1D} Euler-Bernoulli \ac{SPDE}, the actuation of this system enters through the velocity channels, as forces, yet the task is evaluated on the position channels. This results in an evident time-delay between actuation and effect in position space. The soft manipulator was initialized straight and level, with actuators initialized by sampling from a uniform distribution over $[0,w]\times[0,h]$, where $w$ is the width of the system and $h$ is the height. 

The actuation scheme was chosen to mimic a class of manipulators found in biological systems such as cephalopod tentacles or elephant trunks known as muscular hydrostats \cite{walker2005continuum}, in which actuation of incompressible materials conserves the volume of the manipulator. We approximate this property by placing actuators on particles and adjusting the resting length of the tensile connections with adjacent particles. Positive control inputs to the actuator reduce the resting length of horizontal members and increase the resting length of vertical members by the same amount. Negative control inputs have the corresponding opposite effects on the resting length of members. For small actuator inputs, the volume of the manipulator is approximately conserved, resulting in behaviors such as reduction in thickness as the manipulator is axially extended.

The policy network utilized a similar \ac{CNN} architecture as in the \ac{2D} heat experiment, with the entire state space treated as the input. The forces resulting from the the action signal from the policy network being applied to the actuation model act in the velocity space. For this experiment, the convolutional layers of the policy network were initialized with a Xavier~\cite{glorot2010understanding} initialization, while the fully-connected output layer was initialized with zeroes.

As described earlier, a positive control signal on an actuator contracts adjacent horizontal members and expands adjacent vertical members. Thus in order to deflect the tip upwards one unit, the actuators must act in concert to contract the top surface towards the root, and expand the bottom surface towards the tip. Depicted are contour plots of each state of the uncontrolled (top) and controlled (bottom) system at the final time, with optimized actuators depicted in magenta in the velocity states. A video of the uncontrolled and controlled state evolution is available at the link provided \footnote{Video: \href{https://youtu.be/yo48a6JqKE0}{https://youtu.be/yo48a6JqKE0}}. As depicted, the algorithm successfully places actuators in pertinent locations, and is able to achieve the goal state even for exaggerated gravitational forces two orders of magnitude larger than "normal" for our simulation parameters. 

Note that the algorithm co-located four of the ten actuators allocated to the algorithm during optimization, which may be interpreted as an effort by the algorithm to reduce control effort. Such actuator co-location can lead to identification of the minimum set of actuators to achieve a task. The actuator placement in this experiment is quite interesting and merits further analysis. It is challenging to deduce the choices in actuator placement and control, however the authors conjecture that the top left actuator, which is given the largest negative control signal, is mostly responsible for maintaining axial tip location, while the rest of the top surface actuators are responsible for contracting the surface for an upwards deflection. 

These results demonstrate intelligent actuator placements that leverage the system dynamics. The proposed policy and co-design optimization approach is evidently applicable even for nonlinear, 2nd-order, stochastic, \ac{2D} \ac{PDE} systems. The authors intend to continue to  analyze and extend these results further.



\section{Discussion}

Each of the above experiments presents an interesting challenge from the perspective of policy optimization and co-design. The system domains, behavior, and dimensionality are quite varied, and are representative of many of the natural phenomena found in nature and robotics literature. In each of these cases, the optimization parameters are initialized in a random way, so that the optimization does not have a "warm start". The two fully connected and convolutional policy network architectures were rather simple and shallow, did not undergo significant hyper-parameter tuning, and were re-used in all \ac{1D} and all \ac{2D} experiments, respectively. In several of the experiments, optimized actuator locations did not coincide with a human a-priori placement, which in the case of the Nagumo equation resulted in outperformance of the human placement by the algorithm.

The proposed algorithm was successful at joint actuator and policy optimization in each of these cases. Actuator placement optimization through the proposed framework appeared to leverage the dynamics in all of the provided experiments. In practice, the authors found that the presented loss function is equipped with several useful properties. Firstly, it incorporates significant information density, which allows numerous back-propagation gradient paths for both actuator co-design and policy optimization. Secondly, the expectation over rollouts allows a form of exploration of the state space. Finally, exponentiated weighting of trajectory performance allows the loss to clearly differentiate between better state trajectories and worse state trajectories. Indeed rollouts over a system with Cylindrical Wiener noise plays a useful role in the proposed sampling-based optimization scheme, however most experiments were successful with only $50$ rollouts.

The intuition behind the joint policy and actuator co-design optimization presented in this work as opposed to alternating optimization, which is often used for optimization of interdependent variables \cite{boyd2011distributed}, may be explained as follows. It is evident that the specific loss function being used in the proposed approach has simultaneous gradient information for \textit{all} design variables; there is not a separate loss function for policy network performance and a separate loss function for actuator design performance. As such, an alternating approach must either a) only use the dense loss information for policy updates on the inner loop or b) collect gradients with respect to actuator design variables on the inner loop for a large outer loop update. Each option presents obvious issues. The first in essence "wastes" the loss information for a potential actuator design update, while the second potentially sums conflicting gradient information. Instead, the proposed joint optimization approach \textit{leverages} the dense information in the loss function to simultaneously update \textit{all} design variables, where update rates can be simply controlled by the respective learning rates. Thus, one may still prevent actuator variables from updating "too fast" compared to the rate of improvement of the policy variables, or vice versa. 

However, joint optimization is not without fault. Throughout our experiments, RAM usage grows with actuation variables, problem size, time horizon, and the number of policy variables. As we continue to scale further, RAM usage may be a limiting factor. Alternating optimization methods are known to handle this deficiency very well, and accelerated variants of ADMM are an appealing future research direction to tackle scalability to \ac{3D} problem spaces, longer time horizons, deeper networks, and larger sets of actuator co-design optimization variables.

\section{Conclusion}

This work presents a measure-theoretic policy and actuator co-design optimization framework, which was developed in Hilbert spaces for the control of stochastic partial differential equations. Necessary mathematical results for extension to second-order \acp{SPDE} were presented and proved. Novel methods were introduced to decrease computational complexity and increase optimization performance. The resulting loss function was optimized with a popular variant of gradient descent, and the optimization architecture was applied to six simulated \ac{SPDE} experiments that each exhibited unique challenges. The last of which is a biologically inspired model of soft-robotic manipulators with muscular-like actuation, which connects to our goals of establishing capabilities for the further development of soft-body robotics. The results demonstrate that the proposed approach can perform joint policy and actuator co-design optimization on varied complex nonlinear stochastic \ac{PDE} systems.

The presented approach is a new way of performing optimization and can lead to many applications in soft robotics, soft materials, morphing, and continuum mechanics. The results are encouraging to the authors. We plan to continue to develop methods for scalability, explore actuator shape optimization, investigate novel soft-robotic models, and apply the approach to a variety of hard spatio-temporal problems traditionally outside of the realm of robotics research.

\section*{Acknowledgements}
This work was supported by the Army Research Office contract W911NF-20-1-0151. Ethan N. Evans was supported by the SMART scholarship.


\bibliographystyle{ieeetr}
\bibliography{Bibliography}

\end{document}